\documentclass[10pt,twocolumn,a4paper]{article}
\makeatletter
\providecommand{\l@chapter}
\makeatother
\usepackage{times}
\usepackage{cvww}
\usepackage{etex}
\usepackage{url}
\usepackage[numbers]{natbib}
\usepackage{cvpr-abbriv}
\usepackage{epsfig}
\usepackage{graphicx}
\usepackage{amsmath}
\usepackage{amssymb}
\usepackage{array}
\usepackage{tabularx}
\usepackage{url}
\usepackage{graphicx}
\usepackage{amsmath,amssymb} 
\usepackage{amsthm}
\usepackage{amsxtra}
\usepackage{stmaryrd}
\usepackage{comment}
\usepackage[ruled,vlined,linesnumbered,procnumbered,longend]{algorithm2e}

\SetCommentSty{mycommfont}
\usepackage{multirow,rotating}
\usepackage{array}
\usepackage{afterpage}
\usepackage{dblfloatfix}
\usepackage{setspace}
\usepackage{floatrow}
\usepackage{booktabs}
\usepackage{color, colortbl}
\usepackage{tabularx,calc}
\usepackage{tikz}                 
\usetikzlibrary{patterns,fit,external}
\usepackage{pgfplots}             
\usepackage{tocloft}
\usepackage{bbm}
\usepackage[absolute]{textpos}
\usepackage{accents}
\usepackage{units}
\usepackage{xifthen}

\usepackage{times}
\usepackage{tikzscale}
\usepackage{marginnote}
\usepackage[textwidth=2cm,figwidth=\columnwidth,colorinlistoftodos]{todonotes}
\makeatletter
\renewcommand{\todo}[2][]{\tikzexternaldisable\@todo[#1]{#2}\tikzexternalenable}
\makeatother

\newcounter{mycomment} 

\usepackage{booktabs}
\usepackage{multirow}
\usepackage{tabularx}
\usepackage{array}
\usepackage{floatrow}
\usepackage[
  position=bottom,
  font=small,
  labelfont={bf},
]{caption,subfig}
\usepackage{longtable}
\usepackage{tabu} 
\usepackage[most]{tcolorbox}
\usepackage{pbox}
\newlength{\luw}
\newlength{\luh}
\newcommand{\flabels}[1]{%
    \settowidth{\luw}{\begin{tabular}{l}\,#1\,\end{tabular}}%
		\settototalheight{\luh}{%
		\begin{tcolorbox}[width=\luw+5pt, boxsep=1pt, left=0pt, right=0pt, top=0pt, bottom=0pt,nobeforeafter]%
		\begin{tabular}{l}%
		\,#1\,\vphantom{$f^1$fq}%
		\end{tabular}%
		\end{tcolorbox}%
		}%
		\pbox[t][][b]{\textwidth}{%
		\resizebox{!}{0.8\luh}{%
		\begin{tcolorbox}[width=\luw+5pt, boxsep=1pt, left=0pt, right=0pt, top=0pt, bottom=0pt,nobeforeafter]%
		\begin{tabular}{l}%
		\,#1\,\vphantom{$f^1$fq}%
		\end{tabular}%
		\end{tcolorbox}%
		}%
		}%
}%

\newcommand{\labelgraphics}[3][width=0.5\linewidth]{%
\setlength{\fboxsep}{0pt}%
\tikzset{external/export next=false}%
\begin{tikzpicture}%
  \node[inner sep = 0pt] (a) {\includegraphics[#1]{#2}};
  \node[anchor=north west,inner sep=1pt] at (a.north west) {\flabels{#3}};
\end{tikzpicture}%
}%

\usepackage[draft=false, pagebackref=false,breaklinks=true,colorlinks,bookmarks=false]{hyperref}
\hypersetup{hidelinks,colorlinks=false,linkbordercolor={1 1 1}}
\usepackage{xspace}
\DeclareMathAlphabet{\mathcalligra}{T1}{calligra}{m}{n}
\DeclareMathAlphabet{\mathantt}{OT1}{antt}{li}{it}
\DeclareMathAlphabet{\mathpzc}{OT1}{pzc}{m}{it}

\newcommand{\argmin}{\mathop{\rm argmin}}

\DeclareMathOperator{\sign}{sign}

\newcommand{\<}{\langle}
\renewcommand{\>}{\rangle}

\renewcommand{\mid}{\:|\,}

\newcommand{\Real}{\mathbb{R}}

\newcommand{\A}{\mathcal{A}}

\newcommand{\V}{\mathcal{V}}

\newcommand{\E}{\mathcal{E}}

\renewcommand{\L}{\mathcal{L}}

\def\T{^{\mathsf T}}

\def\c{\textsc{c}}

\newcommand{\gray}{\color[rgb]{0.5,0.5,0.5}}
\newcommand{\red}{\color[rgb]{1,0,0}}

\renewcommand*{\paragraph}[1]{\par{\normalsize\bf #1}\ }

\def\mathrlap{\mathpalette\mathrlapinternal} 
\def\mathllap{\mathpalette\mathllapinternal}
\def\mathllapinternal#1#2{\llap{$\mathsurround=0pt#1{#2}$}}
\def\mathrlapinternal#1#2{\rlap{$\mathsurround=0pt#1{#2}$}}

\def\leftbb{\mathrlap{[}\hskip1.3pt[}
\def\rightbb{]\hskip1.36pt\mathllap{]}}

\newcommand{\IF}{\mbox{ \rm if }}

\newcommand{\AND}{\mbox{ \rm and }}
\newcommand{\OTHERWISE}{\mbox{ \rm otherwise}}

\newcommand{\Algorithm}[1]{{Algorithm\,\ref{#1}}}

\newcommand{\Section}[1]{\S\ref{#1}}
\newcommand{\Figure}[1]{{Figure\,\ref{#1}}}



\newcommand{\revisit}[1][]{%
\ifthenelse{\equal{#1}{}}{
\ensuremath{\red \triangle}\xspace}{%
{\ensuremath{\red \rhd}\xspace}%
{\gray #1}%
{\ensuremath{\red \lhd}\xspace}%
}%
}

\def\anchor [#1]#2{%
\phantomsection{}#1\xspace\label{#2}%
\def\arga{#2}%
\global\expandafter\def\csname#2\endcsname{%
\hyperref[#2]{#1}\xspace%
}%
}

\def\codefunction [#1]#2{%
\phantomsection{}\label{#2}{\ttfamily #1\xspace}%
\def\arga{#2}%
\global\expandafter\def\csname#2\endcsname{%
\hyperref[#2]{\ttfamily #1}\xspace%
}%
}

\usepackage{array}
\newcolumntype{L}[1]{>{\raggedright\let\newline\\\arraybackslash\hspace{0pt}}m{#1}}
\newcolumntype{C}[1]{>{\centering\let\newline\\\arraybackslash\hspace{0pt}}m{#1}}
\newcolumntype{R}[1]{>{\raggedleft\let\newline\\\arraybackslash\hspace{0pt}}m{#1}}

\newcommand{\mm}{\mathit{m}}

\SetKwFor{forever}{while true}{}{end while}

\newcommand{\ubar}[1]{\underaccent{\bar}{#1}}

\newcommand{\footlabel}[2]{%
    \addtocounter{footnote}{1}%
    \footnotetext[\thefootnote]{%
        \addtocounter{footnote}{-1}%
        \refstepcounter{footnote}\label{#1}%
        #2%
    }%
    $^{\ref{#1}}$%
}

\newcommand{\footref}[1]{%
    $^{\ref{#1}}$%
}

\def\epsilon{\varepsilon}

\usepackage{aliascnt}
\newlength{\myskip}
  {\begin{list}{\arabic{enumi}.}%
     {\topsep=0in\itemsep=0in\parsep=0pt\partopsep=0in\usecounter{enumi}}%
   }{\end{list}}

\renewenvironment{itemize}%
  {\begin{list}{$\bullet$}%
     {\topsep=0in\itemsep=0pt\parsep=0pt\partopsep=0in\usecounter{itemi}}%
   }{\end{list}\addvspace{0pt}}

\SetAlgoSkip{skipalgomyskip}
\SetAlgoInsideSkip{}

\makeatletter
\let\corollary\@undefined
\let\c@corollary\@undefined
\let\endcorollary\@undefined
\let\definition\@undefined
\let\c@definition\@undefined
\let\enddefinition\@undefined
\let\theorem\@undefined
\let\c@theorem\@undefined
\let\endtheorem\@undefined
\let\lemma\@undefined
\let\c@lemma\@undefined
\let\endlemma\@undefined
\makeatother

\newtheoremstyle{tightItalic}
  {0.5\myskip}
  {0.5\myskip}
  {}
  {}
  {\itshape}
  {.}
  { }
  {}

\newtheoremstyle{tightBf}
  {0.5\myskip}
  {0.5\myskip}
  {}
  {}
  {\bf}
  {.}
  {.5em}
  {}

\newtheoremstyle{tightBBf}
  {0.5\myskip}
  {0.5\myskip}
  {}
  {}
  {\bf}
  {.}
  {.5em}
  {}

\theoremstyle{tightBf}
\newtheorem{theorem}{Theorem}[section]

\newtheorem*{theorem*}{Theorem}
\newaliascnt{corollary}{theorem}%
\newtheorem{corollary}[corollary]{Corollary}
\aliascntresetthe{corollary}

\newaliascnt{definition}{theorem}%
\newtheorem{definition}[definition]{Definition}
\aliascntresetthe{definition}

\newaliascnt{statement}{theorem}%

\aliascntresetthe{statement}

\newaliascnt{lemma}{theorem}%
\newtheorem{lemma}[lemma]{Lemma}
\aliascntresetthe{lemma}

\newaliascnt{example}{theorem}%
\newtheorem{example}[example]{Example}
\aliascntresetthe{example}

\newaliascnt{remark}{theorem}%
\newtheorem*{remark*}{Remark}
\aliascntresetthe{remark}

\newaliascnt{proposition}{theorem}%
\newtheorem{proposition}[proposition]{Proposition}
\aliascntresetthe{proposition}

\newaliascnt{property}{theorem}%

\aliascntresetthe{property}

\theoremstyle{tightBBf}

\newaliascnt{problem}{theorem}%

\aliascntresetthe{problem}

\theoremstyle{tightItalic}


\usepackage{thmtools, thm-restate}
\usepackage[capitalise,nameinlink]{cleveref}
\crefformat{equation}{(#2#1#3)}
\crefname{section}{\S}{\S}

\graphicspath{{figures/}}

\cvwwfinalcopy 

\def\cvwwPaperID{6} 

\ifcvwwfinal\fi
\begin{document}

\title{Solving Dense Image Matching in Real-Time using Discrete-Continuous
Optimization}

\author{Alexander Shekhovtsov, Christian Reinbacher, Gottfried Graber and Thomas Pock
\\
Institute for Computer Graphics and Vision, Graz University of Technology
\\
{\tt\small \{shekhovtsov,reinbacher,graber,pock\}@icg.tugraz.at}
}

\maketitle
\ifcvwwfinal\thispagestyle{fancy}\fi

\begin{abstract}
Dense image matching is a fundamental low-level problem in Computer Vision, which has received tremendous attention from both discrete and continuous optimization communities. 
The goal of this paper is to combine the advantages of discrete and continuous optimization in a coherent framework. We devise a model based on energy minimization, to be optimized by both discrete and continuous algorithms in a consistent way.
In the discrete setting, we propose a novel optimization algorithm that can be massively parallelized.
In the continuous setting we tackle the problem of non-convex regularizers by a formulation based on differences of convex functions.
The resulting hybrid discrete-continuous algorithm can be efficiently accelerated by modern GPUs and we demonstrate its real-time performance for the applications of dense stereo matching and optical flow.
\end{abstract}


\section{Introduction}
The \emph{dense image matching} problem is one of the most basic problems in computer vision: 
The goal is to find matching pixels in two (or more) images. The applications include stereo, optical flow, medical image registration, face recognition~\cite{DBLP:journals/prl/ArashlooK14}, \etc.
Since the matching problem is inherently ill-posed,
typically \emph{optimization} is involved in solving it. We can distinguish two fundamentally
different approaches: discrete and continuous optimization. Whereas discrete approaches (see~\cite{Kappes-2013-benchmark} for a recent comparison)
assign a distinct label to each output pixel, 
continuous approaches try to solve for a \emph{function} using the calculus of variations
\cite{Chambolle2011,Combettes2011,Ochs2014}. Both approaches have
received enormous attention, and there exist state-of-the-art algorithms in both camps:
continuous~\cite{Ranftl2014,Ranftl2011,Sinha2014} and discrete~\cite{Menze2015GCPR,Taniai2014}. 
Due to the specific mathematical tools available to solve the problems (discrete combinatorial optimization vs. continuous calculus of variations), both approaches have distinct advantages and disadvantages.
\begin{figure}[ht!]
  \centering
	\setlength{\fboxsep}{0pt}
	\setlength{\tabcolsep}{0pt}%
	\tabulinesep=0.5pt%
	\renewcommand{\arraystretch}{0}%
	\begin{tabular}{|c|}
	\hline%
  \labelgraphics[width=0.95\columnwidth]{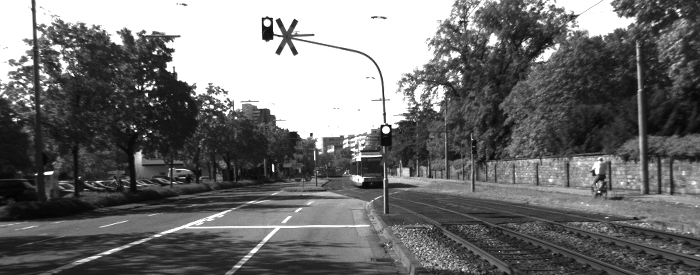}{Input}\\
	\hline%
  \labelgraphics[width=0.95\columnwidth]{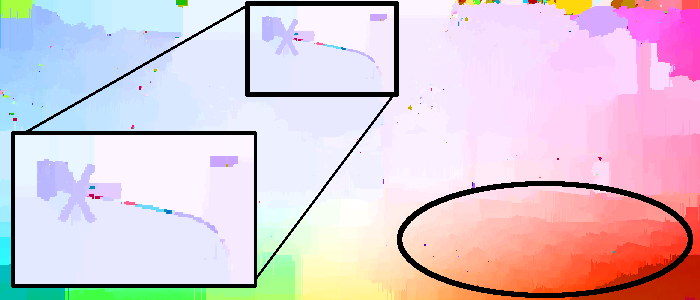}{Discrete}\\
	\hline%
  \labelgraphics[width=0.95\columnwidth]{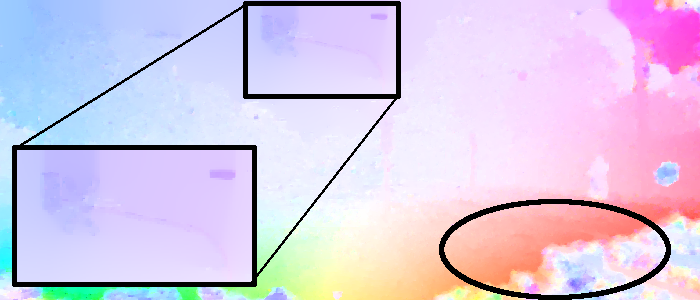}{Continuous}\\
	\hline%
  \labelgraphics[width=0.95\columnwidth]{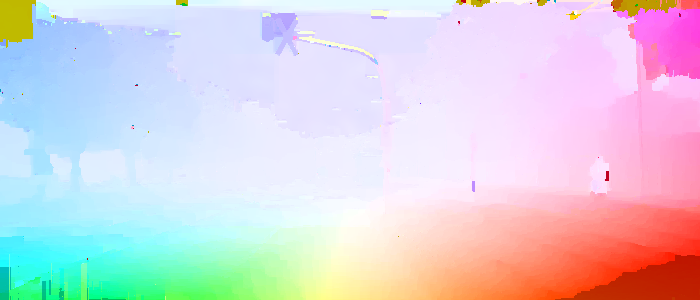}{Combined}\\
	\hline
\end{tabular}\\[5pt]
\begingroup
\resizebox{\linewidth}{!}{
\renewcommand{\arraystretch}{1.0}%
\begin{tabular}{c|ccc}
                & data term &        Large motion\ \ &\ \ Parallelization \\ 
								\hline
    Discrete &  Arbitrary (sampled) &  Easy &            Difficult \\ 
		\hline
Continuous\ \ &\ \ Convex (linearized) &  Difficult &     Easy\\ 
  \end{tabular}
	}
	\endgroup
	\vskip-5pt
  \caption{
	Optical flow problem solved by a purely discrete method, a purely continuous method and the combined method. All methods are as described in this paper, they use the same data term and are run until convergence here.
	In the discrete solution we can see small scale details and sharp motion boundaries but also discretization artifacts. 
	The continuous solution exhibits sub-pixel accuracy (smoothness), but lacks small details and has difficulties with large motions. 
	The combined solution delivers smooth flow fields while retaining many small scale details. 
}
  \label{fig:intro_flow_discrete_cont}
\end{figure}

In this paper, we argue that on a fundamental level the advantages and disadvantages of discrete and continuous optimization for dense matching problems are \emph{complementary} as summarized in \Figure{fig:intro_flow_discrete_cont}.
The previous work combining discrete and continuous optimization primarily used discrete optimization to fuse (find the optimal crossover) of candidate continuous proposals, \eg~\cite{Woodford-08,Taniai2014} (stereo) and~\cite{FusionFlow} (flow). The latter additionally performs local continuous optimization of the so-found solution. 
Many works also alternate between continuous and discrete optimizations, addressing a Mumford-Shah-like model, \eg, \cite{Brox-06}. 
Similarly to~\cite{FusionFlow} we introduce a continuous energy which is optimized using a combined method. However, we work with a full (non-local) discretization of this model and propose new parallel optimization methods.

The basic difference in discrete and continuous approaches lies in the handling of the data term. The data term is a measure how well the
solution (\ie value of a pixel) fits the underlying measurement (\ie input images). In the discrete
setting, the solution takes discrete labels, and hence the number of labels is finite. Typically the
data cost is precomputed for all possible labels. The discrete optimization then uses the data cost
to find the optimal label for each pixel according to a suitable model in an energy minimization
framework. We point out that due to the \emph{sampling} in both label space and spatial domain, the
discrete algorithm has access to the full information at every step. \Ie it deals with a {\em global
optimization} model and in some lucky cases can find a globally optimal solution to it or provide an
approximation ratio or partial optimality guarantees~\cite{SSS-15-IRI}.

In the continuous setting, the solution is a continuous function. This means it is not possible to
precompute the data cost; an infinite number of solutions would require infinite amount of
memory. More importantly, the data cost is a non-convex function stemming from the similarity measure between the images.
In order to make the optimization problem tractable, a popular approach is the linearization of the data cost. 
However, this introduces a range of new problems, namely the inability to deal with large motions due to the fact that the linearization is valid only in a small neighborhood around the linearization point. 
Most continuous methods relying on linearization therefore use a coarse-to-fine framework in an attempt to overcome this problem \cite{Brox04}. 
One exception is a recent work~\cite{KolmogorovPR15}, which can handle piece-wise linear data terms and truncated TV regularization.
\par
Our goal in this paper is to combine the advantages of both approaches, as well as real-time
performance, which imposes tough constraints on both methods resulting in a number of challenges:
\paragraph{Challenges} 
The discrete optimization method needs to be highly parallel and able to {\em couple} the noisy / ambiguous data over large areas. The continuous energy should be a refinement of the discrete energy so that we can evaluate the two-phase optimization in terms of a single energy function. The continuous method needs to handle robust (truncated) regularization terms.
\paragraph{Contribution} Towards the posed challenges, we propose: i) a new method for the discrete problem, working in the dual (\ie making equivalent changes of the data cost volume), in parallel on multiple chains; 
ii) a continuous optimization method, reducing non-convex regularizers to a primal-dual method with non-linear operators~\cite{Valkonen2014};
iii) an efficient implementation of both methods on GPU and proof of concept experiments showing advantages of the combined approach.
%


\section{Method}\label{sec:method}
%
In this section we will describe our two-step approach to the dense image matching problem. To combine the previously discussed advantages of discrete and continuous optimization methods it is essential to minimize the same energy in both optimization methods. Starting from a continuous energy formulation in \cref{sec:model}, we first show how to discretize the energy in \cref{sec:discretization} and subsequently minimize it using a novel discrete parallel block coordinate descent, described in \cref{sec:discrete_optimization}. The output of this algorithm will be the input to a refinement method which is posed as a continuous optimization problem, solved by a non-linear primal-dual algorithm described in \cref{sec:continuous_refinement}.
\subsection{Model}\label{sec:model}
Let us formally define the dense image matching problem to be addressed by the discrete-continuous optimization approach.
In both formulations we consider that the image domain is a discrete {\em set of pixels} $\V$. The continuous formulation has continuous ranged variables $u = (u^k_{i} \in \Real \mid k = 1,\dots d,\ i\in \V)$, where $d = 1, 2$  for stereo / flow, respectively. The matching problem is formulated as
\begin{equation}\label{eqn:mrf_formulation}
\min_{u\in U} \Big[ E(u) = D(u) + R(A u) \Big],
\end{equation}
where $U = \mathbb{R}^{d \times \mathcal{V}}$; $D$ is the {\em data term} and $R(A u)$ is a {\em regularizer} ($A$ is a linear operator explained below). The discrete formulation will quantize variable ranges.
%
%
%
%
%
\par
\paragraph{Data Term} We assume $D(u) = \sum_{i\in \V} D_i(u_i)$, where $D_i \colon \Real^d \to \Real$ encodes the deviation of $u_i$ from some underlying measurement. A usual choice for dense image matching are robust filters like Census Transform or Normalized Cross Correlation, computed on a small window around a pixel. This data term is non-convex in $u$ and piecewise linear. In the discrete setting, the data term is sampled at discrete locations, in the continuous setting, the data term is convexified by linearizing or approximating it around the current solution. The details will be described in the respective sections.

\paragraph{Regularization Term} The regularizer encodes properties of the solution of the energy minimization like local smoothness or preservation of sharp edges. The choice of this term is crucial in practice, since the data term may be unreliable or uninformative in large areas of dense matching problems. We assume 
\begin{align}\label{regularizer-form}
R(A u) = \sum_{ij\in\E} \omega_{ij} \sum_{k = 1}^{d} r ( (A u^k )_{ij} ),
\end{align}
where $\E\subset \V\times\V$ is the set of {\em edges}, \ie, pairs of neighboring pixels; linear operator $A \colon \Real^\V  \to \Real^{\E} \colon u^k \mapsto ( u^k_i - u^k_j \in \Real \mid \forall ij\in\mathcal{E})$ essentially computes gradients along the edges in $\E$ for the solution dimension $k$; the gradients are penalized by the {\em penalty function} $r\colon \Real \to \Real$ and $\omega_{ij}$ are image dependent per-edge strength weights, reducing the penalty around sharp edges. 
Our particular choice for the penalty function $r$ is depicted in \cref{fig:regularizer_cont}.  We chose to use a truncated norm which has shown to be robust against noise that one typically encounters in dense matching problems. It generalizes truncated Total Variation in the continuous setting. In the discrete setting it generalizes the $P1$-$P2$ penalty model~\cite{hirschmuller2011semi}, Potts model and the truncated linear model.
\begin{figure}[ht]
  \begin{center}
  \includegraphics[width=0.8\columnwidth,height=0.5\columnwidth]{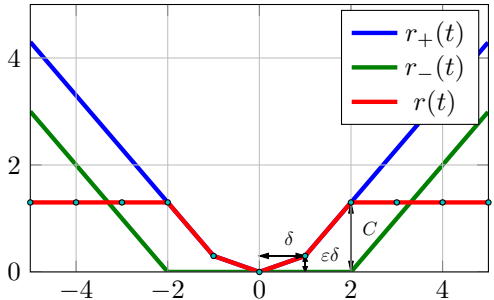}
  \caption{
	Regularizer function $r$. In our continuous optimization method it is decomposed into a difference of convex functions $r_{+} - r_{-}$. For the discrete optimization it is sampled at label locations depicted as dots.
	}
  \label{fig:regularizer_cont}
  \end{center}
\end{figure}


\subsection{Discrete Formulation}\label{sec:discretization}
In the discrete representation we will use the following formalism.
To a continuous variable $u_i$ we associate a discrete variable $x_i \in \L$.
The discrete label space $\L$ can be chosen to our convenience as long as it has the desired number of elements, denoted $K$. We let $\L$ to be vectors in $\{0,1\}^K$ with exactly one component equal $1$ (the 1-hot encoding of natural numbers from 1 to $K$). 
For $f_i \in \Real^K$ we denote $f_i(x_i) = \<f_i, x_i\> = f_i\T x_i$ and for $f_{ij} \in \Real^{K \times K}$ we denote $f_{ij}(x_i, x_j) = x_i \T f_{ij} x_j$. Let $f = (f_w \mid w \in \V \cup \E)$ denote the energy {\em cost vector}. The {\em energy function} corresponding to the cost vector $f$ is given by 
\begin{align}\label{energy-pairwise}
f(x) = \sum_{i\in\V}f_i(x_i) + \sum_{ij\in\E}f_{ij}(x_i,x_j).
\end{align}
Whenever we need to refer to $f$ as a function and not as the cost vector, we will always use the argument notation, \eg $f(x) \geq g(x)$ is different from $f \geq g$.
\par
Energy function $f$ that can be written as $\sum_{i}f_i(x_i) = \<f, x\>$ is called {\em modular}, {\em separable} or {\em linear}. Formally, all components $f_{ij}$ of $f$ are identically zero. If $f_{ij}$ is non-zero only for a subgraph of $(\V,\E)$ which is a set of chains, we say that $f$ is a {\em chain}.
\par
The {\em discrete energy minimization} problem is defined as
\begin{align}\label{emin-discrete}
\min_{x \in \L^\V} f(x).
\end{align}
\paragraph{Stereo} 
We discretize a range of disparities and let $u(x) \in \Real^\V$ denote the continuous solution corresponding to the {\em labeling} $x$. We set $f_i(x_i) = D_i(u(x_i))$ and $f_{ij}(x_i,x_j) = \omega_{ij}r((A u(x))_{ij})$. 
\paragraph{Flow} Discretization of the flow is somewhat more challenging. Since $u_i$ is a 2D vector, assuming large displacements, discretizing all combinations is not tractable. Instead, components $u_i^{1}$ and $u_i^{2}$ can be represented as separate discrete variables $x_{i^1}, x_{i^2}$ , where $(i^1, i^2)$ is a pair of nodes duplicating $i$, leading to the {\em decomposed} formulation~\cite{Shekhovtsov-Kovtun-Hlavac-08-CVIU}. To retain the pairwise energy form~\eqref{energy-pairwise}, this approach assigns the data terms $D_i(u_i)$ to a pairwise cost $f_{i^1 i^2}(x_{i^1}, x_{i^2})$ and the regularization is imposed on each layer of variables $(x_{i^1} \mid i\in\V)$ and $(x_{i^2} \mid i\in \V)$ separately. To this end, we tested a yet simpler representation, in which we assign optimistic data costs, given by
\begin{subequations}\label{flow decoupled costs}
\begin{align}
&\textstyle f_{i^1}(x_{i^1}) = \min_{x_{i^2}} D_i(x_{i^1},x_{i^2}),\\
&\textstyle f_{i^2}(x_{i^2}) = \min_{x_{i^1}} D_i(x_{i^1},x_{i^2}),
\end{align}
\end{subequations}
where $D_i(x_{i^1},x_{i^2})$ is the discretized data cost, and regularize in each layer individually. This makes the two layers fully decouple into, essentially, a two independent stereo-like problems. At the same time, the coupled scheme~\cite{Shekhovtsov-Kovtun-Hlavac-08-CVIU}, passing messages between the two layers, differs merely in recomputing~\eqref{flow decoupled costs} for a reparametrized data costs in a loop. Our simplification then is not a principled limitation but an intermediate step.
%
%
%
\subsection{Discrete Optimization}\label{sec:discrete_optimization}
In this section we give an overview of a new method under development addressing problem~\eqref{emin-discrete} through its LP-relaxation dual.
In real-time applications like stereo and flow there seem to be a demand in methods performing fast approximate discrete optimization, preferably well-parallelizable. It has motivated a significant research. The challenge may sound as ``{\em best solution in a limited time budget}''.
\par
Well-performing methods, from local to global, range from cost volume filtering~\cite{Hosni-13}, semi-global matching (SGM)~\cite{hirschmuller2011semi} (has been implemented  in GPU and FPGA~\cite{Banz-10}), dynamic programming on spanning trees adjusting the cost volume~\cite{Bleyer-08} and more-global matching (MGM)~\cite{Facciolo-15} to the sequential dual block coordinates methods, such as TRW-S~\cite{Kolmogorov-06-convergent-pami}. Despite being called sequential, TRW-S exhibits a fair amount of parallelism in its computation dependency graph, which is exploited in the parallel GPU/FPGA implementations~\cite{ChoiR12,hurkatfast-15}.
At the same time SGM has been interpreted~\cite{Drory-14} as a single step of parallel TRW algorithm~\cite{Wainwright-MAP} developed for solving the dual. MGM goes further in this direction, resembling even more the structure of a dual solver: it combines together more messages but in a heuristic fashion and introducing more computation dependencies, in fact similar to TRW-S. It appears that all these approaches go somehow in the direction of a fast processing of the dual.
\par
We propose a new dual update scheme, which:
i) is a monotonous block-coordinate ascent; ii) performs as good as TRW-S for an equal number of iterations while having a comparable iteration cost;
and iii) offers more parallelism, better mapping to current massively parallel compute architectures.
Thus it bridges the gap between highly parallel heuristics and the best ``sequential'' dual methods without compromising on the speed and performance.
%
\par
On a higher level, the method is most easily presented in the dual decomposition framework. 
For clarity, let us consider a decomposition into two subproblems only (horizontal and vertical chains).
Consider minimizing the energy function $E(x)$ that separates as 
\begin{align}\label{energy-chain-chain}
E(x) = f(x) + g(x),
\end{align}
where $f,\ g \colon \L^\V \to \Real$ are chains. 
\paragraph{Primal Majorize-Minimize} 
Even before introducing the dual, we can propose applying the majorize-minimize method (a well-known optimization technique) to the primal problem in the form~\eqref{energy-chain-chain}. It is instructive for the subsequent presentation of the dual method and has an intriguing connection to it, which we do not yet fully understand.
\begin{definition} A modular function $\bar f$ is a {\em majorant} (upper bound) of $f$ if
$(\forall x) \ \ \bar f (x) \geq f(x)$, symbolically $\bar f \succeq f$. A modular minorant $\ubar f$ of $f$ is defined similarly.\footnote{$\ubar f$ reads ``f underbar''.}
\end{definition}
Noting that minimizing a chain function plus a modular function is easy, one could straightforwardly propose \Algorithm{PMM}, which alternates between majorizing one of $f$ or $g$ by a modular function and minimizing the resulting chain problem $\bar f + g$ (resp. $f + \bar g$). We are not aware of this approach being evaluated before. Somewhat novel, the sum of two chain functions is employed rather than, say, difference of submodular~\cite{Narasimhan-05}, but the principle is the same. To ensure monotonicity of the algorithm we need to pick a majorant $\bar f$ of $f$ which is exact in the current primal solution $x^k$ as in \cref{PMM-step1}. 
Then $f(x^{k+1}) + g(x^{k+1})  \leq \bar f(x^{k+1}) + g(x^{k+1}) \leq \bar f(x^{k}) + g(x^{k}) = f(x^{k}) + g(x^{k})$. Steps \ref{PMM-step3}-\ref{PMM-step4} are completely similar. 
\Algorithm{PMM} has the following properties:
\begin{itemize}
\item primal monotonous;
\item parallel, since, \eg, $\min_x (\bar f + g)(x)$ decouples over all vertical chains;
\item uses more information about subproblem $f$ than just the optimal solution (as in most primal block-coordinate schemes: ICM, alternating lines, \etc.).
\end{itemize}
The performance of this method highly depends on the strategy of choosing majorants. This will be also the main question to address in the dual setting.

\begin{algorithm}[tb]
\KwIn{Initial primal point $x^k$\;}
\KwOut{New primal point $x^{k+2}$\;}
$\bar f \succeq f$, $\bar f(x^k) = f(x^k)$\label{PMM-step1}\tcc*{Majorize}
$x^{k+1} \in \argmin\limits_{x} (\bar f + g)(x)$\label{PMM-step2}\tcc*{Minimize}
$\bar g \succeq g$, $\bar g(x^{k+1}) = g(x^{k+1})$\label{PMM-step3}\tcc*{Majorize}
$x^{k+2} \in \argmin\limits_{x} (f + \bar g)(x)$\label{PMM-step4}\tcc*{Minimize}
\caption{\protect\anchor[{\texttt{Primal\_MM}}]{PMM}}
\end{algorithm}

\paragraph{Dual Decomposition} Minimization of~\eqref{energy-chain-chain} can be written as
\begin{align}\label{DD-decompose}
\min_{x^1 = x^2} f(x^1) + g(x^2).
\end{align}
Introducing a vector of Lagrange multipliers $\lambda \in \Real^{\L \times \V}$ for the constraint $x^1 = x^2$, we get the Lagrange dual problem: 
\begin{align}\label{DD-dual}
\max_{\lambda} \hskip-0.5ex \Big[ \hskip-0.5ex \underbrace{\min_{x}\big( f(x) + \<\lambda, x\>\big)}_{D^1(\lambda)} + \underbrace{\min_{x}\big( g(x) - \<\lambda, x\>\big)}_{D^2(\lambda)} \Big].
\end{align}
The so-called {\em slave} problems $D^1(\lambda)$ and $D^2(\lambda)$ have the form of minimizing an energy function with a data cost modified by $\lambda$. The goal of the {\em master problem}~\eqref{DD-dual} is to balance the data cost between the slave problems such that their solutions agree. 
The slave problems are minima of finitely many functions linear in $\lambda$, the objective of the master problem~\eqref{DD-dual} $D(\lambda) = D^1(\lambda) + D^2(\lambda)$ is thus a concave piece-wise linear function. Problem~\eqref{DD-dual} is a concave maximization. 
However, since $x$ was taking values in a discrete space, there is only a weak duality: \eqref{DD-decompose} $\geq$ \eqref{DD-dual}. It is known that~\eqref{DD-dual} can be written as a linear program (LP), which is as difficult in terms of computation complexity as a general LP~\cite{Prusa-Werner-15-Universality}.
\paragraph{Dual Minorize-Maximize} In the dual, which is a maximization problem, we will speak of a minorize-maximize method. 
The setting is similar to the primal. We can {\em efficiently} maximize $D^1$, $D^2$ but not $D^1 + D^2$. Suppose we have an initial dual point $\lambda^0$ and let $x^0 \in \argmin_{x}(f + \lambda^0)(x)$ be a solution to the slave subproblem $D^1$, that is, $D^1(\lambda^0) = f(x^0)+\lambda^0(x^0)$.

\begin{proposition}\label{P:dual-minorant}
Let $\ubar f$ be a modular minorant of $f$ exact in $x^0$ and such that $\ubar f + \lambda^0 \geq D^1(\lambda^0)$ (component-wise). 
Then the function $\ubar{D}^1(\lambda) = \min_{x}(\ubar f+\lambda)(x)$ is a minorant of $D^1(\lambda)$ exact at $\lambda = \lambda^0$.
\end{proposition}
\begin{proof}
Since $\ubar f (x) \leq f(x)$ for all $x$ it follows that $\min_{x}(\ubar{f}+\lambda)(x) \leq \min_{x}(f+\lambda)(x)$ for all $\lambda$ and therefore $\ubar{D}^1$ is a minorant of $D^1$. 
Next, on one hand we have $\ubar{D}^1(\lambda^0) \leq D^1(\lambda^0)$ and on the other, $D^1(\lambda^0) \leq (\ubar f+\lambda^0)(x)$ for all $x$ and thus $D^1(\lambda^0) \leq \ubar{D}^1(\lambda^0)$.
\end{proof}
We have constructed a minorant of $D^1$ which is itself a (simple) piece-wise linear concave function. The maximization step of the minorize-maximize is to solve
\begin{align}\label{DMM:max-phase}
\max_\lambda (\ubar D^1(\lambda) + D^2(\lambda)).
\end{align}
\begin{proposition}
$\lambda^* = {-}\ubar f$ is a solution to ~\eqref{DMM:max-phase}.
\end{proposition}
\begin{proof}
Substituting $\lambda^*$ into the objective~\eqref{DMM:max-phase} we obtain $\ubar D^1(\lambda^*) + D^2(\lambda^*) = \min_{x}(\ubar f - \ubar f)(x) + D^2({-}\ubar f) = \min_{x}(\ubar f + g)(x)$.
This value is the maximum because $\ubar D^1(\lambda) + D^2(\lambda)  = \min_{x}(\ubar f+\lambda)(x) + \min_{x}(g - \lambda)(x)  \leq \min_{x}(\ubar f+\lambda + g - \lambda)(x) = \min_{x}(\ubar f + g)(x)$.
\end{proof}
Note, for the dual point $\lambda = -\ubar f$, in order to construct a minorant of $D^2$ (similarly to \cref{P:dual-minorant}) we need to find a solution to the second slave problem, 
\begin{align}
x^1 \in \argmin(g - \lambda)(x) = \argmin(\ubar f + g)(x).
\end{align}
We obtain \Algorithm{DMM} with the following properties:
\begin{itemize}
\item It builds the sequence of dual points given by $\lambda^{2t} = \ubar g^{2t}, \lambda^{2t+1} = -\ubar f^{2t+1}$ and the dual objective does not decrease on each step;
\item The minimization subproblems and minorants are decoupled (can be solved in parallel) for all horizontal (resp. vertical) chains;
\item When provided good minorants (see below) the algorithm has same fixed points as TRW-S~\cite{Kolmogorov-06-convergent-pami};
\item Updating only a single component $\lambda_i$ for a pixel $i$ is a monotonous step as well, therefore the algorithm is a {\em parallel block-coordinate ascent}.
\end{itemize}
%
\begin{algorithm}[tb]
\KwIn{Initial dual point $\ubar g^k$\;}
\KwOut{New dual point $\ubar g^{k+2}$\;}
$x^{k} \in \argmin_{x} (f + \ubar g^{k} )(x)$\label{DMM-step1}\tcc*{Minimize}
\tcc{Minorize}
$\ubar f^{k+1} \preceq f$, $\ubar f^{k+1}(x^k) = f(x^k)$, $\ubar f^{k+1} + \ubar g^{k} \geq f(x^k) + \ubar g^{k} (x^k)$\label{DMM-step2}\;
$x^{k+1} \in \argmin_{x} (\ubar f^{k+1} + g)(x)$\label{DMM-step3}\tcc*{Minimize}
\tcc{Minorize}
$\ubar g^{k+2} \preceq g$, $\ubar g^{k+2}(x^{k+1}) = g(x^{k+1})$, $\ubar f^{k+1} + \ubar g^{k+2} \geq \ubar f^{k+1}(x^{k+1}) + g (x^{k+1})$\;
\caption{\protect\anchor[{\texttt{Dual\_MM}}]{DMM}}
\end{algorithm}
Notice also that \DMM and \PMM are very similar, nearly up to replacing minorants with majorants. The sequence $\{E(x^k)\}_k$ is monotonous in \Algorithm{PMM} but not in \Algorithm{DMM}.
\paragraph{Good and Fast Minorants}
The choice of the minorant in \DMM is non-trivial as there are many, which makes it sort of a secrete ingredient. \Figure{fig:alg-uniform} illustrates two of the possible choices. The {\em naive} minorant for a chain problem $f+\lambda$ is constructed by calculating its min-marginals and dividing by chain length to ensure that the simultaneous step is monotonous (\cf tree block update algorithm of~\citet[Fig. 1]{Sontag-09}). The {\em uniform} minorant is found through the optimization procedure that tries to build the tightest modular lower bound, by increasing uniformly all components that are not yet tight. The details are given in~\Section{sec:discrete-details}. In practice, we build fast minorants, which try to approximate the uniform one using fast message passing operations. Parallelization of decoupled chains allowed us to achieve an implementation which, while having the same number of memory accesses as TRW-S (including messages / dual variables), saturates the GPU memory bandwidth, $\sim$ 230GB/s.\footnote{This is about 10 times faster than reported for FPGA implementation~\cite{ChoiR12} of TRW-S.} This allows to perform 5 iterations of \Algorithm{DMM} for an image 512$\times$512 and 64 labels at the rate of about 30 fps.

\begin{figure}[tb]
\centering
\includegraphics[width=\linewidth]{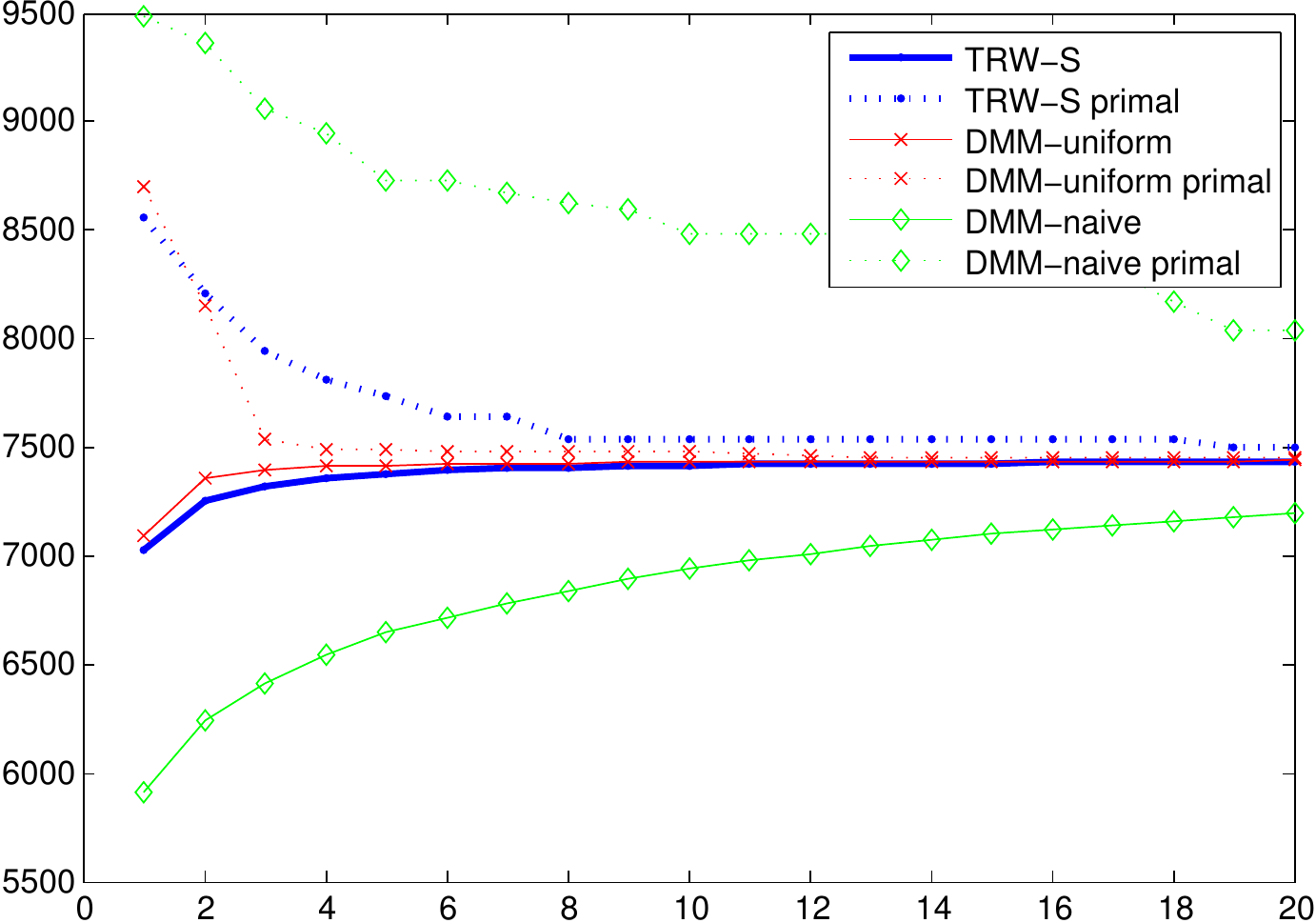}
\caption{Lower bounds and best primal solutions by TRW-S and by \DMM with a naive and a uniform minorants. The problem is a small crop from stereo of size 40$\times 40$, 16 labels, truncated linear regularization. On the $x$-axis one iteration is a forward-backward pass of TRW-S \vs iteration of \DMM (equal number of updates per pixel). With a good choice of minorant, \DMM can perform even better than the sequential baseline in terms of iterations. Parallelizing it can be expected to give a direct speedup. 
\label{fig:alg-uniform}}
\end{figure}


\subsection{Continuous Refinement}
\label{sec:continuous_refinement}
In this section we describe the continuous refinement method, which is based on variational energy minimization. The goal of this step is to refine the output of the optimization method described in \Cref{sec:discrete_optimization} which is discrete in label-space.

To that end, it is important to minimize the same energy in both formulations. Considering the
optimization problem in \cref{eqn:mrf_formulation}, we are seeking to minimize a non-convex,
truncated norm together with a non-convex data term. For clarity, let us write down the problem again:
\begin{equation}\label{eqn:compact_formulation}
\min_{u\in U}  D(u) + R(A u) .
\end{equation}

\paragraph{Non-Convex Primal-Dual}
Efficient algorithms exist to solve \cref{eqn:compact_formulation} in case both $D(u)$ and $R(Au)$ are convex (but possibly non-smooth), \eg the primal-dual solver of Chambolle and Pock \cite{Chambolle2011}. \citet{KolmogorovPR15} solves \cref{eqn:compact_formulation} for a truncated total variation regularizer using a splitting into horizontal and vertical 1D problems and applying \cite{Chambolle2011} to the Lagrangian function. Here we will use a recently proposed extension to \cite{Chambolle2011} by \citet{Valkonen2014}. 
He considers problems of the form $\min_x \mathcal{G}(x) + \mathcal{F}(\mathcal{A}(x))$, \ie of the same structure as \cref{eqn:compact_formulation}, where $\mathcal{G}$ and $\mathcal{F}$ are convex, $\mathcal{G}$ is differentiable and $\mathcal{A}(u)$ is a twice differentiable but possibly non-linear operator.
In the primal-dual formulation, the problem is written as
\begin{equation}
\label{eqn:cont_nonlin_pd}
  \min_x \max_y \big[ \mathcal{G}(x) + \langle \mathcal{A}(x),y\rangle -\mathcal{F}^*(y) \big],
\end{equation}
where $^*$ is the convex conjugate. Valkonen 
proposes the following modified primal-dual hybrid gradient method: 
\begin{subequations}\label{eqn:cont_iterates}
\begin{align}
  x^{k+1} =& (I+\tau \partial \mathcal{G})^{-1}(x^k-\tau \left[\nabla \mathcal{A}(x^k)\right]\T y^k)\\
  y^{k+1} =& (I+\sigma \partial \mathcal{F}^*)^{-1}(y^k+\sigma\mathcal{A}(2 x^{k+1}- x^k)) .
\end{align}
\end{subequations}

\paragraph{Reformulation}
In order to apply method \cite{Valkonen2014}, we will reformulate the non-convex problem \cref{eqn:compact_formulation} to the form~\eqref{eqn:cont_nonlin_pd}. We start by formulating the regularizer $R(Au)$ as a \emph{difference of convex functions}: 
$R(Au)=R_{+}(Au)-R_{-}(Au)$, where $R_{+}$ and $R_{-}$ are convex. 
The primal-dual formulation of \eqref{eqn:compact_formulation} then reads
\begin{align}\label{eqn:compact_formulation_duals}
   \min_u \Big[ &\max_p (\langle Au,p\rangle - R_{+}^*(p)) \\
         & \notag + \max_q (\langle Au,q\rangle - R_{-}^*(q)) + D(u)\Big].
\end{align}
Because $\min_x -f(x) = -\max_x f(x)$, \eqref{eqn:compact_formulation_duals} equals 
\begin{align}
  \min_u \Big[ & \max_p ( \langle Au,p\rangle - R_{+}^*(p) ) + \\
         & \notag + \min_q (-\langle Au,q\rangle + R_{-}^*(q)) + D(u) \Big].
\end{align}
Grouping terms we arrive at
\begin{equation}\label{eqn:cont_problem}
	\min_{u,q} \max_p \Big[ \langle Au,p - q\rangle - R_{+}^*(p) + R_{-}^*(q) + D(u) \Big].
\end{equation}
The problem now arises in minimizing the bilinear term $\<A u, q\>$ in \eqref{eqn:cont_problem} in both $u$ and $q$. 
We thus move this term into the nonlinear operator $\mathcal{A}(x)$ and rewrite \cref{eqn:cont_problem} as
\begin{multline}
  \underbrace{\min_{u,q}}_{x} \underbrace{\max_{p,d=1}}_{y}
  \left\langle 
    \underbrace{\begin{bmatrix}
      A u\\
      - \langle A u,q\rangle
    \end{bmatrix}}_{\mathcal{A}(x)},
    \begin{bmatrix}
      p\\
      d
    \end{bmatrix}
  \right\rangle + \underbrace{R_{-}^*(q) + D(u)}_{\mathcal{G}(x)}\\ - 
  \underbrace{R_{+}^*(p)}_{\mathcal{F}^*(y)}
\end{multline}
by introducing a dummy variable $d=1$.

\paragraph{Implementation Details}
The gradient of $\A$ needed by iterates~\eqref{eqn:cont_iterates} is given by
\begin{equation}
 \nabla \mathcal{A}(x) = 
 \begin{bmatrix}
  A & 0 \\
  -A\T q & -A u
 \end{bmatrix}.
\end{equation}
The regularization function $r$ is represented as a difference of two convex functions (see \Figure{fig:regularizer_cont}):
\begin{align}\label{r-decompose}
r (t) =  r_{\epsilon,\delta}(t) - r_{0,(C + \delta -\epsilon \delta)}(t),
\end{align}
where
\vskip-1em
\begin{equation}
 r_{\alpha,\beta}(t)=\begin{cases}
                            \alpha |t| & \text{if } |t|\leq \beta\\
                            |t|-\beta(1-\alpha) & \text{else}
                          \end{cases}
\end{equation}
is convex for $\alpha \leq 1$.
Convex functions $R_{+}(Au)$ and $R_{-}(Au)$ are defined by decomposition~\eqref{r-decompose} and~\eqref{regularizer-form}.

To compute the proximal map $(I+\sigma \partial \mathcal{F}^*)^{-1}(\hat{y})$ we first need the convex conjugate of $\omega_{ij} r_{\alpha,\beta}(t)$. It is given by $(\omega_{ij} r_{\alpha,\beta})^*(t^*) =$
\begin{equation}
  \begin{cases}
    \max(0,\beta|t^*|- \omega_{ij}\alpha \beta) & \text{if } \alpha<|t^*|<\omega_{ij}\\
    \infty & \text{else}
  \end{cases}\ .
\end{equation}
The proximal map for $(\omega_{ij}r_{\alpha,\beta})^*$ at $t^* \in \Real$ is given by 
$\bar t = \operatorname{clamp}(\pm \omega_{ij}, t')$, where $\operatorname{clamp}(\pm \omega_{ij},\cdot)$ denotes a clamping to the interval $[-\omega_{ij},\omega_{ij}]$ and 
\begin{equation}
\label{eqn:pprox}
  t' = \left\{\begin{aligned}
    & t^* & \text{if } |t^*| \leq \alpha \omega_{ij}\\
		& \max(\alpha\omega_{ij},|t^*|{-}\beta\sigma)\sign( t^*)\hskip-1cm
		& \text{else}.
	\end{aligned}\right.
\end{equation}
%
Proximal map $(I+\sigma \partial \mathcal{F}^*)^{-1}(\hat{y})$ is calculated by applying expression \cref{eqn:pprox} component-wise to $\hat y$. The proximal map $(I+\tau \partial \mathcal{G})^{-1}$ depends on the choice of the data term $D(u)$ and will thus be defined in \cref{sec:applications}.



\section{Applications}
\label{sec:applications}
\subsection{Stereo Reconstruction}
\label{sec:app_stereo}
For the problem of estimating depth from two images, we look at a setup of two calibrated and synchronized cameras. We assume that the input images to our method have been rectified according to the calibration parameters of the cameras. 
We aim to minimize the energy \cref{eqn:mrf_formulation} where $u$ encodes the disparity in $x$-direction. The data term measures the data fidelity between images $I_1$ and $I_2$, warped by the disparity field $u$. As a data term we use the Census Transform \cite{Zabih1994} computed on a small local patch in each image. The cost is given by the pixel-wise Hamming distance on the transformed images.
$D(u)$ is non-convex in the argument $u$ which makes the optimization problem in \cref{eqn:mrf_formulation} intractable in general.

We start by minimizing \cref{eqn:mrf_formulation} using the discrete method (\Section{sec:discrete_optimization}) in order to obtain an initial solution $\mathring{u}$. We approximate the data term around the current point $\mathring{u}$ by a piecewise linear convex function $\tilde{D}(u) = $
\begin{equation}
  D(\mathring{u}) + \delta_{[\mathring{u}-h,\mathring{u}+h]}(u) + 
  \begin{cases}
    s_1 (u-\mathring{u}) & \IF u\leq \mathring{u}\\
    s_2 (u-\mathring{u}) & \OTHERWISE
  \end{cases}
\end{equation}
with  $s_{1} = \frac{D(\mathring{u}+h) - D(\mathring{u})}{h}$ and $s_{2} =
\frac{D(\mathring{u})-D(\mathring{u}+h)}{h}$ for a small $h$. To ensure convexity, we set $s_1 = s_2 =
\frac{s_1+s_2}{2}$ if $s_2<s_1$. The indicator function $\delta$ is added to ensure that the solution
stays within $\mathring{u}\pm h$ where the approximation is valid.
We then apply the continuous method (\Section{sec:continuous_refinement}). The proximal map $\bar u = (I+\tau \partial \mathcal{G})^{-1}(\hat u)$ needed by the algorithm \cref{eqn:cont_iterates} for the approximated data term expresses as the pointwise soft-thresholding 
\begin{equation*}
	\bar{u}_i = \operatorname{clamp}\left (\mathring{u}_i\pm h, \hat{u}_i - 
  \begin{cases}
    \tau s_{1,i} & \IF \hat{u}_i>\mathring{u}_i+\tau s_{1,i}\\
    \tau s_{2,i} & \IF \hat{u}_i<\mathring{u}_i+\tau s_{2,i}\\
    0 & \OTHERWISE
  \end{cases}\right )\\
\end{equation*}
In practice, the minimization has to be embedded in a {\em warping framework}: after optimizing for $n$ iterations, the data term is approximated anew at the current solution $u$.
\subsection{Optical Flow}
\label{sec:app_flow}
The optical flow problem for two images $I_1,I_2$ is posed again as model \cref{eqn:mrf_formulation}. 
In contrast to stereo estimation, we now have $u_i\in\mathbb{R}^2$
encoding the flow vector. 
For the discrete optimization step (\Section{sec:discrete_optimization}) the flow problem is decoupled into two independent stereo-like problems as discussed in \Section{sec:discretization}.

For the continuous refinement step, the main problem is again the non-convexity of the data term. 
Instead of a convex approximation with two linear slopes
we build a quadratic approximation, now in 2D, following~\cite{Werlberger2010}. 
The approximated data term reads $\tilde D_i(u_i) = \delta_{[\mathring{u}_i-h,\mathring{u}_i+h]}(u_i)+$
\begin{equation}
  \label{eq:quadfit}
  D_i(\mathring{u}_i) + L_i\T(u_i-\mathring{u}_i)+\frac{1}{2}(u_i-\mathring{u}_i)\T Q_i (u_i-\mathring{u}_i),
\end{equation}
where $L_i\in\mathbb{R}^2$ and $Q_i\in\mathbb{R}^{2\times 2}$ are finite difference approximations
of the gradient and the Hessian with stepsize $h$. Convexity of \cref{eq:quadfit} is ensured by retaining only positive-semidefinite part of $Q_i$ as in~\cite{Werlberger2010}.
%
%
The proximal map $\bar u = (I+\tau\partial\mathcal{G})^{-1}(\hat u)$ for data term \cref{eq:quadfit} is given point-wise by 
\begin{equation}
  \label{eq:prox_flow}
	\bar u_i^k = \operatorname{clamp}\left(\mathring{u}_i^k \pm h,
    \frac{\hat u_i^k+\tau(Q_{i}\mathring{u}_i-L_{i})^k}{1+\tau L_{i}^k}\right ).
\end{equation}
Optimizing~\cref{eqn:mrf_formulation} is then performed as proposed in~\Section{sec:continuous_refinement}. 



\section{Experiments}

\subsection{Stereo Reconstruction}
We evaluate our proposed real-time stereo method on datasets where Ground-Truth data is available as well as on images captured using a commercially available stereo camera. \subsubsection{Influence of Truncated Regularizer}
We begin by comparing the proposed method to a simplified version that does not use a truncated norm as regularizer but a standard Total Variation. We show the effect of this change in \cref{fig:comp_global}, where one can observe much sharper edges, when using a robust norm in the regularization term. On the downside it is more sensitive to outliers, which however can be removed in a post-processing step like a two-side consistency check.

\begin{figure}[tb]
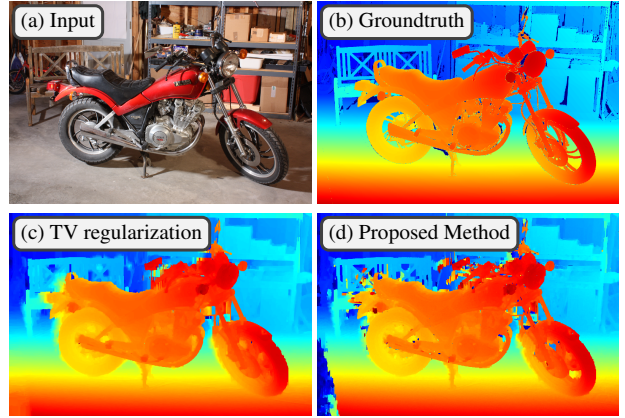

  \begin{center}
	\begin{tabular}{C{0.5\linewidth}C{0.5\linewidth}}
	\labelgraphics[width=0.98\linewidth]{middlebury_stereo/input}{(a) Input}&
	\labelgraphics[width=0.98\linewidth]{middlebury_stereo/gt}{(b) Groundtruth}\\
	\labelgraphics[width=0.98\linewidth]{middlebury_stereo/notrunk}{(c) TV regularization}&
	\labelgraphics[width=0.98\linewidth]{middlebury_stereo/slp_ttv}{(d) Proposed Method}
	\end{tabular}
  \caption{Influence of the robust regularizer in the continuous refinement on stereo reconstruction quality.}
  \label{fig:comp_global}
  \end{center}
\end{figure}

\subsubsection{Live Dense Reconstruction}
\label{sec:live_reconstruction}
To show the performance of our stereo matching method in a real live setting, we look at the task of creating a live dense reconstruction from a set of depth images. To that end, we are using a reimplementation of \emph{KinectFusion} proposed by Newcombe et al. \cite{Newcombe2011} together with the output of our method. This method was originally designed to be used with the RGBD output of a Microsoft Kinect and tracks the 6 DOF position of the camera in real-time. For the purpose of this experiment we replace the Kinect with a Point Grey Bumblebee2 stereo camera. \emph{KinectFusion} can only handle relatively small camera movements between images, so a high framerate is essential. We set the parameters to our method to achieve a compromise between highest quality and a framerate of $\approx 4-5$ fps: camera resolution $640\times 480$, 128 disparities, 4 iterations of \DMM, 5 warps and 40 iterations per warp of the continuous refinement. 

\paragraph{Influence of Continuous Refinement}
The first stage of our reconstruction method, \DMM, already delivers high quality disparity images that include details on fine structures and depth discontinuities that are nicely aligned with edges in the image. In this experiment we want to show the influence of the second stage, the continuous refinement, on the reconstruction quality of KinectFusion. To that end we mount the camera on a tripod and collect 300 depthmaps live from our full method and 300 frames with the continuous refinement switched off. By switching off the camera tracking, the final reconstruction will show us the artifacts produced by the stereo method. \Cref{fig:quadfit} depicts the result of this comparison. One can easily see that the output of the discrete method contains fine details, but suffers from staircasing artifacts on slanted surfaces due to the integer solution. The increase in quality due to the refinement stage can be especially seen on far away objects, where a disparity step of 1 pixel is not enough to capture smooth surfaces.

\begin{figure}[tb]
  \begin{center}
	\begin{tabular}{C{0.5\linewidth}C{0.5\linewidth}}
	\labelgraphics[width=0.98\linewidth]{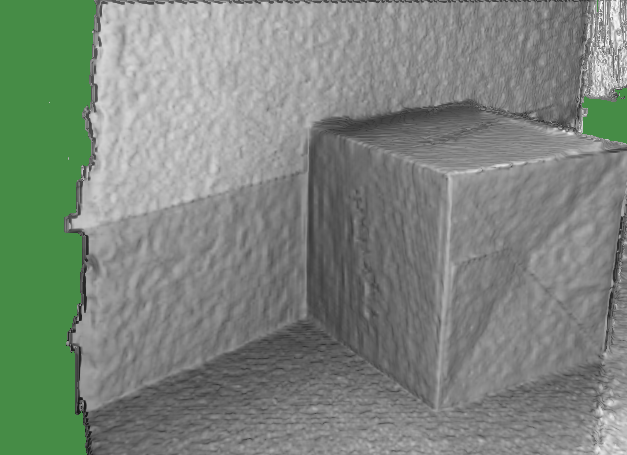}{(a) Refinement}&
	\labelgraphics[width=0.98\linewidth]{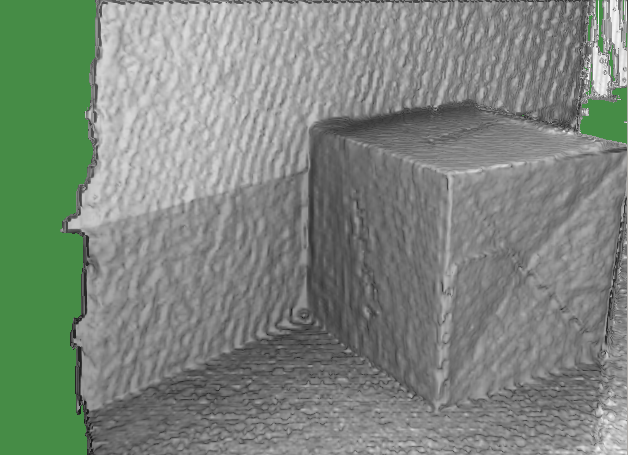}{(b) No Refinement}
	\end{tabular}
  \caption{Influence of continuous refinement on the reconstruction quality of
  KinectFusion.}
  \label{fig:quadfit}
  \end{center}
\end{figure}

\paragraph{Timing}
To show the influence of the individual steps in our stereo method on runtime, we break down the total time of $\approx 140$ ms per frame in \Cref{tab:timing}. Those timings have been achieved using a PC with 32 GB RAM with a NVidia 980GTX, running Linux.

\begin{table}[t]
\begin{center}
	\begin{tabularx}{\columnwidth}{*{3}{X}r@{}}
	  \toprule
		{\small\textbf{Cost Vol.}} & 
		{\small\textbf{Discrete}} & {\small\textbf{Cont. Ref.}} & 
		{\small\textbf{Total}} \\
	  \midrule
		27 ms & 73 ms  & 39 ms & \textbf{139 ms} \\
		\bottomrule
	\end{tabularx}
	\caption{Runtime analysis of the individual components of our stereo matching
	method. Details regarding computing hardware and parameters are in the text. In case of the full left-right check procedure the total computation time doubles.}
	\label{tab:timing}
\end{center}
\end{table}

\begin{figure}[t]
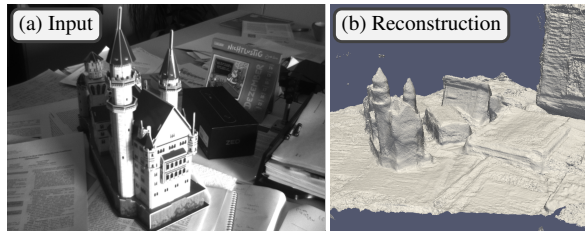

  \centering
	\setlength{\tabcolsep}{0pt}
	\setlength{\fboxsep}{0pt}
	\begin{tabular}{cc}
	\labelgraphics[height=0.37\linewidth]{live_dense/input_left}{(a) Input}&\,%
	\labelgraphics[height=0.37\linewidth]{live_dense/output}{(b) Reconstruction}
	\end{tabular}
  \caption{Qualitative result of reconstructing a desktop scene using KinectFusion\footref{sup}.}
  \label{fig:qualitative_stereo}
\end{figure}

\paragraph{Qualitative Results}
To give an impression about the quality of the generated depthmaps and the speed of our method, we run our full algorithm and aim to reconstruct a desktop scene with a size of $1\times 1\times 1$ meters and show some renderings in \cref{fig:qualitative_stereo}. To better visualize the quality of the geometry, the model is rendered without texture\footlabel{sup}{We point the interested reader to a video that shows the reconstruction pipeline in real-time: \url{http://gpu4vision.icg.tugraz.at/videos/cvww16.mp4}}.

\subsection{Optical Flow}

In this section we show preliminary results of our algorithm applied to optical flow.
A further improvement in quality can be expected by exploiting the coupled scheme~\cite{Shekhovtsov-Kovtun-Hlavac-08-CVIU} in the discrete optimization, as discussed in \cref{sec:discretization}.
%
As
depicted in \Figure{fig:results_flow}, our method is able to deliver reasonable results on a
variety of input images. We deliberately chose scenes that contain large motion as well as small
scale objects, to highlight the strengths of the discrete-continuous approach. 
For comparison, we use a state-of-the-art purely continuous variational optical flow algorithm \cite{Werlberger2012}. 
The runtime of our method is $2$s for an image of size $640\times480$.

\begin{figure}[tb]
  \begin{center}
	\setlength{\tabcolsep}{0pt}
	\begin{tabular}{ccc}
		\labelgraphics[width=0.31\linewidth]{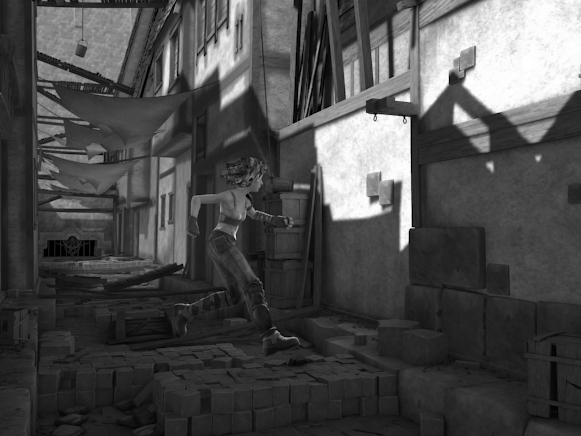}{Inputs}&\,%
		\includegraphics[width=0.31\linewidth]{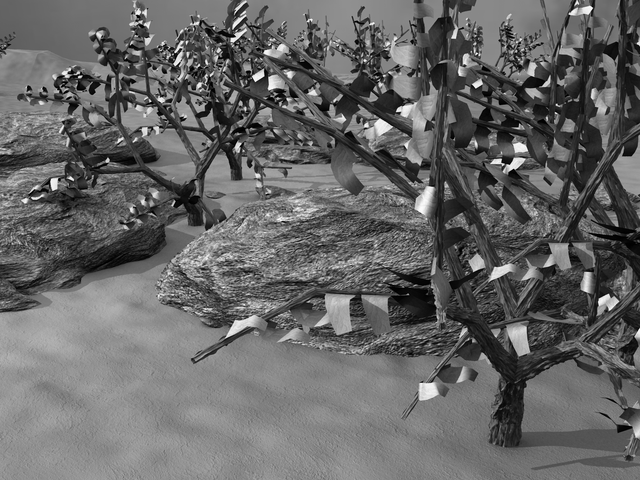}&\,%
		\includegraphics[width=0.31\linewidth]{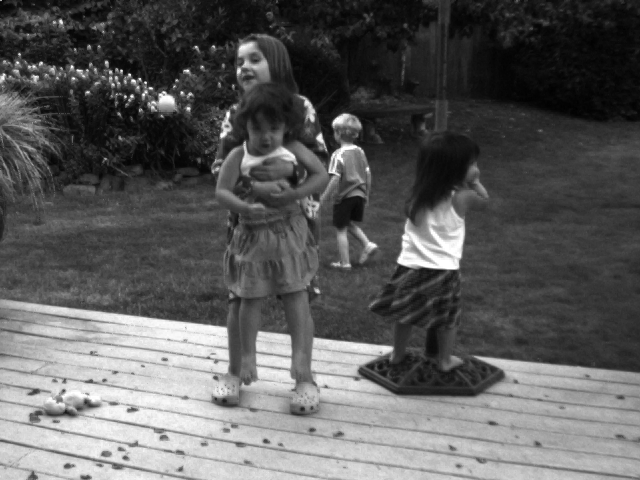}%
	\end{tabular}
	\begin{tabular}{cc}
		\labelgraphics[width=0.47\columnwidth]{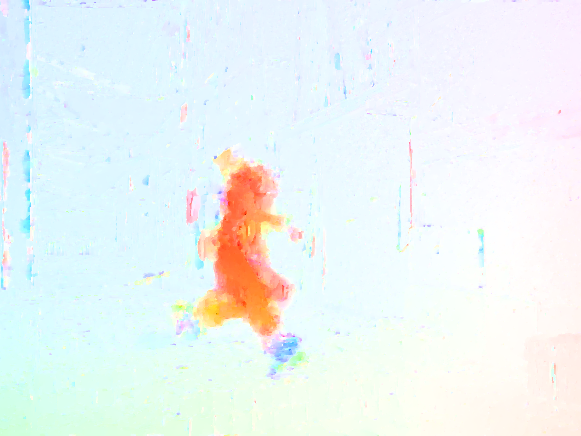}{~\citet{Werlberger2012}}&\ %
		\labelgraphics[width=0.47\columnwidth]{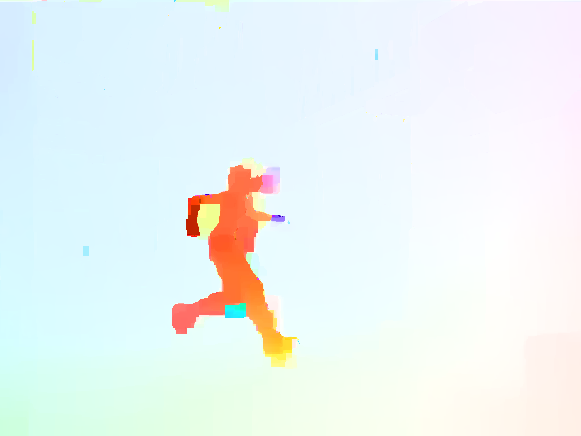}{Combined}\\
		\includegraphics[width=0.47\columnwidth]{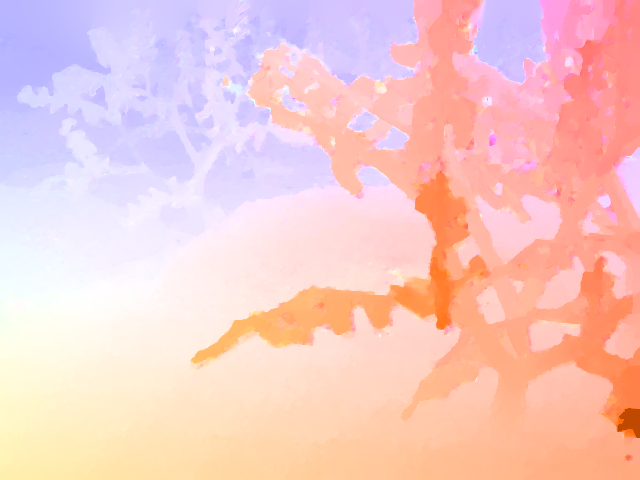}&\ %
		\includegraphics[width=0.47\columnwidth]{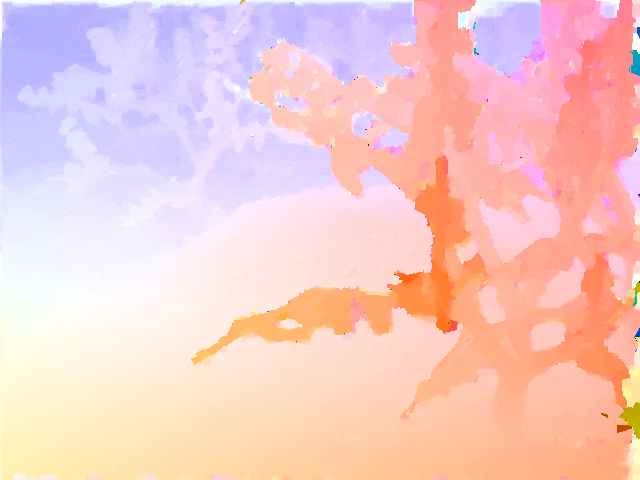}\\
		\includegraphics[width=0.47\columnwidth]{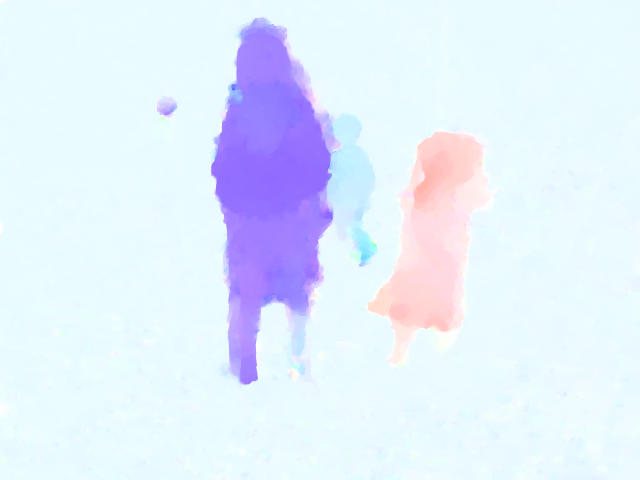}&\ %
		\includegraphics[width=0.47\columnwidth]{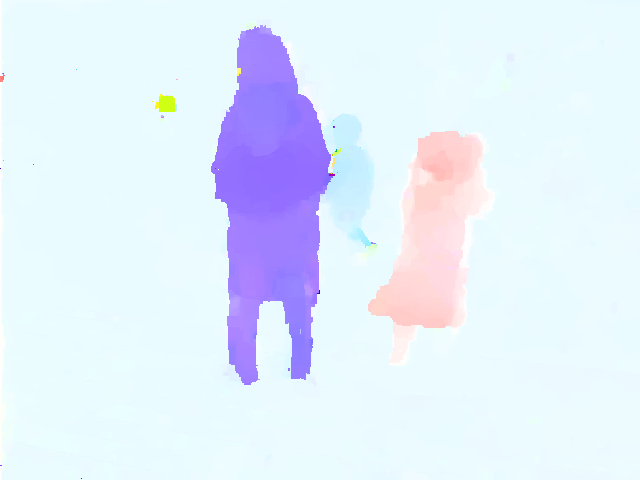}
\end{tabular}		
    \caption{Subjective comparison of variational approach~\cite{Werlberger2012} (left) with our combined method (right).
		Top row show input images, one from a pair. 
			Both methods use the same data term. Parameters of both algorithms have been tuned by hand to deliver good results. Note that for~\cite{Werlberger2012} it is often impossible to get sharp motion boundaries as well as small scale details, despite a very strong data term (\eg artifacts in left image, first row).}
    \label{fig:results_flow}
  \end{center}
\end{figure}


\section{Conclusion}
The current results demonstrate that it is feasible to solve dense image matching problems using global optimization methods with a good quality in real time.
We have proposed a highly parallel discrete method, which even when executed sequentially, is competitive with the best sequential methods. As a dual method, we believe, it has a potential to smoothly handle more complex models in the dual decomposition framework and is in theory applicable to general graphical models.
When the solution is sufficiently localized, continuous representation increases the accuracy of the model as well as optimization speed.
In the continuous optimization, we experimented with non-convex models and showed a reduction allowing to handle them with the help of a recent non-linear primal-dual method. This in turn allowed to speak of a global model to be solved by a discrete-continuous optimization.
\par
Ideally, we would like to achieve a method, which, when given enough time, produces an accurate solution, and in the real time setting gives a robust result. We plan further to improve on the model. A vast literature on the topic suggest that modeling occlusions and using planar hypothesis can be very helpful. At the same time, we are interested in a tighter coupling of discrete and continuous optimization towards a globally optimal solution.

\ifcvwwfinal{
\section*{Acknowledgements}
This work was supported by the research initiative Mobile Vision with funding
from the AIT and the Austrian Federal Ministry of Science, Research and Economy
HRSM programme (BGBl. II Nr. 292/2012).
}
\fi

{\small
\bibliographystyle{apalike}
\bibliography{bib/strings,bib/discrete,bib/books,bib/continuous,bib/suppl}
}

\appendix
\section{Details of Dual MM}\label{sec:discrete-details}
In this section we specify details regarding computation of minorants in \DMM. The minorants are computed using message passing and we'll also need the notion of min-marginals.
\subsection{Min-Marginals and Message Passing}
\begin{definition}
For cost vector $f$ its {\em min-marginal} at node $i$ is the function $m^f \colon \L \to \Real$ given by
\begin{align}
\mm^f(x_i) = \min_{x_{\V \backslash i} } f(x).
\end{align}
\end{definition}
Function $\mm^f(x_i)$ is a min projection of $f(x)$ onto $x_i$ only. Given the choice of $x_i$, it returns the cost of the best labeling in $f$ that passes through $x_i$. For a chain problem it can be computed using dynamic programming. Let us assume that the nodes $\V$ are enumerated in the order of the chain and $\E = \{(i, i+1) \mid i=1\dots |\V|-1\}$. We then need to compute: 
{\em left min-marginals}: $\varphi_{i-1,i}(x_i) := $
\begin{align}
\min_{x_{1,\dots i-1} } \sum_{i' < i}f_{i'}(x_{i'}) + \sum_{i'j' \in \E \mid i' < i }f_{i'j'}(x_{i'},x_{j'});
\end{align}
and {\em right min-marginals}: $\varphi_{i+1,i}(x_i) := $
\begin{align}
\min_{x_{i+1,\dots |V|} } \sum_{i' > i}f_{i'}(x_{i'}) + \sum_{i'j' \in \E \mid i' \geq i }f_{i'j'}(x_{i'},x_{j'}).
\end{align}
These values for all $ij\in\E$, $x_i,x_j\in\L$ can be computed dynamically (recursively). After that, the min-marginal $\mm^f(x_i)$ expresses as
\begin{align}\label{min-marg-expression}
\mm^f(x_i) = f_i(x_i) + \varphi_{i-1,i}(x_i) + \varphi_{i+1,i}(x_i).
\end{align}
\par
TRW-S method~\cite{Kolmogorov-06-convergent-pami} can be derived as selecting one node $i$ at a time and maximizing~\eqref{DD-dual} with respect to $\lambda_i$ only.
For the two slave problems in~\eqref{DD-dual} TRW-S needs to compute min-marginals $\mm^{f+\lambda}(x_i)$ and $\mm^{g-\lambda}(x_i)$. A (non-unique) optimal choice for $\lambda_i$ would be to ensure that
\begin{align}
\mm^{f+\lambda}(x_i) = \mm^{g-\lambda}(x_i) \ \ \forall x_i \in \L
\end{align}
by setting 
\begin{align}\label{TRW-S-update}
\lambda_i := \lambda_i + (\mm^{g-\lambda}(x_i) - \mm^{f+\lambda}(x_i))/2.
\end{align}
\par
If $i$ and $j$ are two nodes in a chain $f+\lambda$ then performing the update of $\lambda_i$ changes the min-marginal at $j$ and vice-versa. The updates must be implemented sequentially or otherwise one gets a non-monotonous behavior and the method may fail to converge (see~\cite{Kolmogorov-06-convergent-pami}). 
\par
TRW-S gains its efficiency in that after the update~\eqref{TRW-S-update}, the min-marginal at a neighboring node can be recomputed by a single step of dynamic programming. Let the neighboring node be $j = i+1$. The expression for the right min-marginal at $j$ remains correct and the expression for left min-marginal is updated using its recurrent expression $\varphi_{ij}(x_{j}) := $
\begin{align}\label{msg-pass}
\min_{x_i} \big[ \varphi_{i-1, i}(x_{i}) + f_{i}(x_{i}) + f_{ij}(x_{i},x_{j})\big],
\end{align}
also known as {\em message passing}. Then min-marginal at $j$ becomes available through~\eqref{min-marg-expression}.
\par
It is possible to perform update~\eqref{TRW-S-update} in parallel by scaling down the step size by the number of variables (or the length of the chain). This is equivalent to decomposing a chain $f$ into $n$ copies with costs $f/n$ so that they contribute one for each node $i$ with a min-marginal $\mm^f(x_i) / n$. Confer to the parallel tree block update algorithm of~\citet[Fig. 1]{Sontag-09}). However, the gain from the palatalization does not pay off the decrease in the step size.
\subsection{Slacks}
In the following we will also use the term slack. Shortly, it is explained as follows. The dual problem~\eqref{DD-dual} can be written as a linear program, see \eg,~\cite{Werner-PAMI07}. Dual inequality constraints in that program can satisfied as equalities, in which case they are tight, or they can be satisfied as strict inequalities in which case there is a {\em slack}. Equivalent reparametrization of the problem (change of the dual variables) can propagate a slack from one constraint (corresponding to a label-node pair) to another one. If all constraints in a group becomes non-tight, their minimum slack can be subtracted and increments the lower bound. Since for a chain problem the LP relaxation is tight, the maximum slack that can be concentrated in a label-node equals the corresponding min-marginal.
\subsection{Good Minoratns}
\begin{definition}A modular minorant $\lambda$ of $f$ is {\em maximal} if there is no other modular minorant $\lambda' \geq \lambda$ such that $\lambda'(x) > \lambda(x)$ for some $x$.
\end{definition}
\begin{lemma}
For a maximal minorant $\lambda$ of $f$ all min-marginals of $f - \lambda$ are identically zero.
\end{lemma}
\begin{proof}
Since $\lambda$ is a minorant, min-marginals $m_i(x_i) = \min_{x_{\V \backslash i}}[f(x)-\lambda(x)]$ are non-negative.
Assume for contradiction that $\exists i$, $\exists x_i$ such that $m_i(x_i) > 0$. Clearly, $\lambda'(x) := \lambda(x) + m_i(x_i)$ is also a minorant and $\lambda' > \lambda$.
\end{proof}
%
%
%
\par
Even using maximal minorants, the \Algorithm{DMM} can get stuck in fixed points which do not satisfy weak tree agreement~\cite{Kolmogorov-06-convergent-pami}, \eg suboptimal even in the class of message passing algorithms. 
Consider the following example of a minorant leading to a poor fixed point.
\begin{example}Consider a model in \Figure{fig:stuck} with two labels and strong Ising interactions ensuring that the optimal labeling is uniform. If we select minorants that just takes the unary term, without redistributing it along horizontal or vertical chains, the lower bound will not increase. For example, for the horizontal chain $(v_1, v_2)$, the minorant $(1,\ 0)$ (displayed values correspond to $\lambda_v(1) - \lambda_v(2)$).
This minorant is maximal, but it does not propagate the information available in $v_1$ to $v_2$ for the exchange with the vertical chain $(v_2,v_4)$.
\begin{figure}[th]
\includegraphics[width=0.5\linewidth]{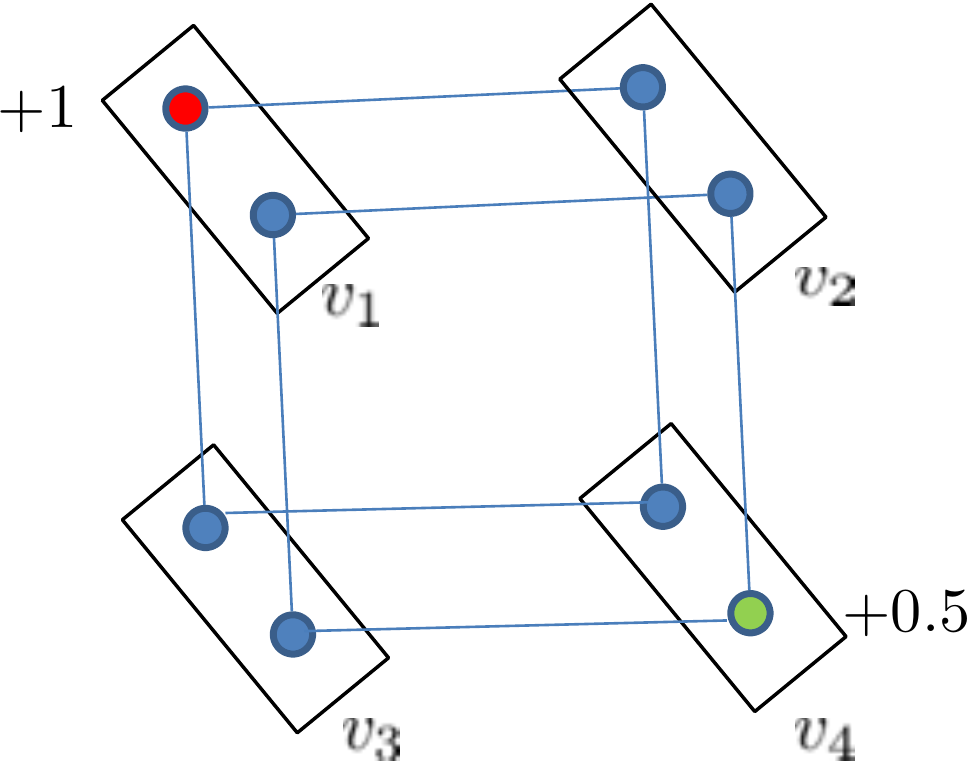}
\caption{Example minorize-minimize stucks with a minorant that does not redistribute slack.\label{fig:stuck}}
\end{figure}
\end{example}
\subsubsection{Uniform Minorants}
Dual algorithms, by dividing the slacks between subproblems ensure that there is always a non-zero fraction of it (depending on the choice of weights in the scheme) propagated along each chain. We need a minorant, which will expose in every variable what is the preferable solution for the subproblem. We can even try to treat all variables uniformly. The practical strategy proposed below is motivated by the following.
\begin{proposition}Let $f^* = \min_{x}f(x)$ and let $O_u$ be the support set of all optimal solutions $x_u^*$ in $u\in\V$. Consider the minorant $\lambda$ given by $\lambda_u(x_u) = \varepsilon (1 - O_u)$ and maximizing $\varepsilon$:
\begin{align}\label{max-epsilon}
&\max \{ \varepsilon \mid (\forall x) \ \varepsilon \<1-O, x\>  \leq f(x) \}.
\end{align}
\end{proposition}
The above minorant assigns cost $\varepsilon$ to all labels but those in the set of optimal solutions. If the optimal solution $x^*$ is unique, it takes the form $\lambda = \varepsilon (1 - x^*)$. This minorant corresponds to the direction of the subgradient method and $\varepsilon$ determines the step size which ensures monotonicity. 
However it is not maximal. In $f - \lambda$ there still remains a lot of slack that can be useful when exchanging to the other problem. 
It is possible to consider $f - \lambda$ again. If we have solved~\eqref{max-epsilon}, it will necessarily have a larger set of optimal solutions. 
We can search for a maximal $\varepsilon_1$ that can be subtracted from all non-optimal label-nodes in $f - \lambda$ and so on. 
The algorithm is specified as \Algorithm{A:uniform-minorant}. 
\begin{algorithm}[tb]
\KwIn{Chain subproblem $f$\;}
\KwOut{Minorant $\lambda$\;}
$\lambda :=0$\;
\forever{}{
Compute min-marginals $m$ of $f-\lambda$\;
\lIf{$m = 0$}{\Return{$\lambda$}}
Let $O := \leftbb m = 0 \rightbb$, the support set of optimal solutions of $m-\lambda$\;
Find $\max \{ \varepsilon \mid (\forall x) \ \varepsilon \<1-O, x\>  \leq (f-\lambda)(x) \}$\label{Uniform-opt-problem}\;
Let $\lambda := \lambda + \varepsilon (1-O)$\;
}
\caption{Maximal Uniform Minorant\label{A:uniform-minorant}}
\end{algorithm}
\par
The optimization problem in \cref{Uniform-opt-problem} can be solved using the minimum ratio cycle algorithm of \citet{Lawler-66}. We search for a path with a minimum ratio of the cost given by $(f-\lambda)(x)$ to the number of selected labels with non-zero min-marginals given by $\<1-O,x\>$. This algorithm is rather efficient, however \Algorithm{A:uniform-minorant} it is still too costly and not well-suited for a parallel implementation. We will not use this method in practice directly, rather it establishes a sound baseline that can be compared to.
\par
The resulting minorant $\lambda$ is maximal and uniform in the following sense. 
\begin{lemma}Let $m$ be the vector of min-marginals of $f$. The uniform minorant $\lambda$ found by \Algorithm{A:uniform-minorant} satisfies 
\begin{align}
\lambda \geq m/n,
\end{align}
where $n$ is the length of the longest chain in $f$.
\end{lemma}
\begin{proof}
This is ensured by~\Algorithm{A:uniform-minorant} as in each step the increment $\varepsilon$ results from dividing the min-marginal by $\<1-O, x\>$ which is at most the length of the chain.
\end{proof}
In fact, when the chain is strongly correlated, the minorant will approach $m/n$ and we cannot do better than that. However, if the correlation is not as strong the minorant becomes tighter, and in the limit of zero pairwise interactions there holds $\lambda = m$. In a sense the minorant computes ``decorrellated'' min-marginals.
\par
The next example illustrates uniform minorants and steps of the algorithm.
%
%
%
%
\begin{example} Consider a chain model with the following data unary cost entries  (3 labels, 6 nodes):
\par
\setlength{\tabcolsep}{1ex}
\begin{tabular}{llllll}
 0 &  0 &  1 &  0 &  0 &  8\\
 9 &  7 &  0 &  3 &  2 &  8\\
 7 &  3 &  6 &  9 &  1 &  0\\
\end{tabular}
\newline\noindent
The regularization is a Potts model with cost $f_{uv}(x_u,x_v) = 1 \leftbb x_u \neq x_v \rightbb$. Min-marginals of the problem and iteration of~\Algorithm{A:uniform-minorant} are ilustrated in \Figure{fig:example-min-marg}. At the first iteration the constructed minorant is
\par
\begin{tabular}{llllll}
 0 &  0 &  0 &  0 &  0 &  1\\
 1 &  1 &  1 &  1 &  1 &  1\\
 1 &  1 &  1 &  1 &  1 &  0\\
\end{tabular}
\newline\noindent
And the final minorant is:
\par
\begin{tabular}{llllll}
 0 &  0 &  0 &  0 &  0 &  7\\
 8 &  7 &  1 &  2 &  2 &  7\\
 6 &  4 &  6 &  7 &  1 &  0\\
\end{tabular}
\newline\noindent
The minorant follows min-marginals (first plot in \Figure{fig:example-min-marg}), because the interaction strength is relatively weak and min-marginals are nearly independent. If we increase interaction strength to 5, we find the following min-marginals and minorant, respectively:
\newline\noindent
\renewcommand{\arraystretch}{1.2}
\begin{tabular}{llllll}
  0 &   0 &   0 &   0 &   0 &   3\\
 14 &  15 &   8 &   8 &   7 &   8\\
 12 &  13 &  15 &  10 &   1 &   0\\
\hline
  0 &   0 &   0 &   0 &   0 &   3\\
5.5 & 5.5 &   3 &   3 &   3 &   3\\
4.75 & 4.75 & 4.75 & 4.75 &   1 &   0\\
\end{tabular}
\par\noindent
It is seen that in this case min-marginals are correlated and only a fraction can be drained in parallel. The uniform approach automatically divides the cost equally between strongly correlated labels.
\begin{figure}[th]
\centering
\begin{tabular}{c}
\small (a)\ \begin{tabular}{c}\includegraphics[width=0.9\linewidth]{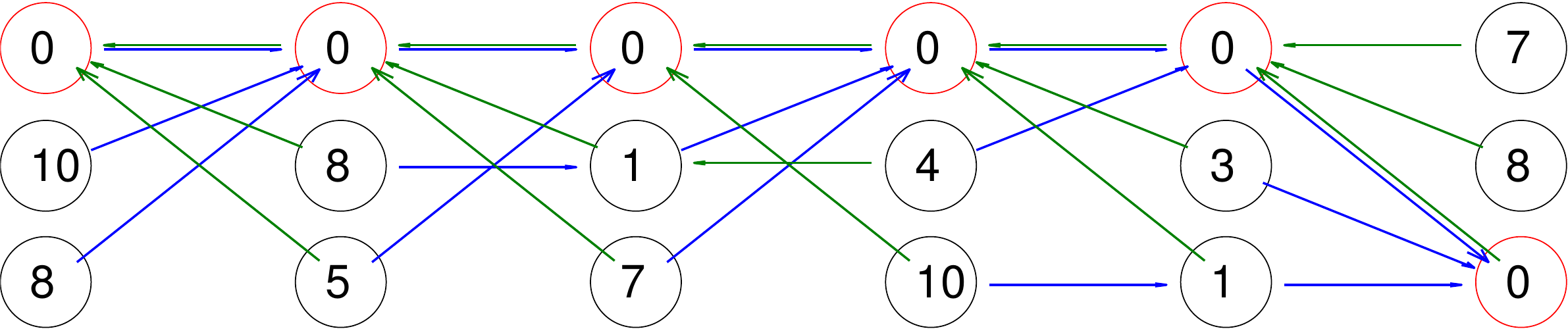}\end{tabular}\\
\small (b)\ \begin{tabular}{c}\includegraphics[width=0.9\linewidth]{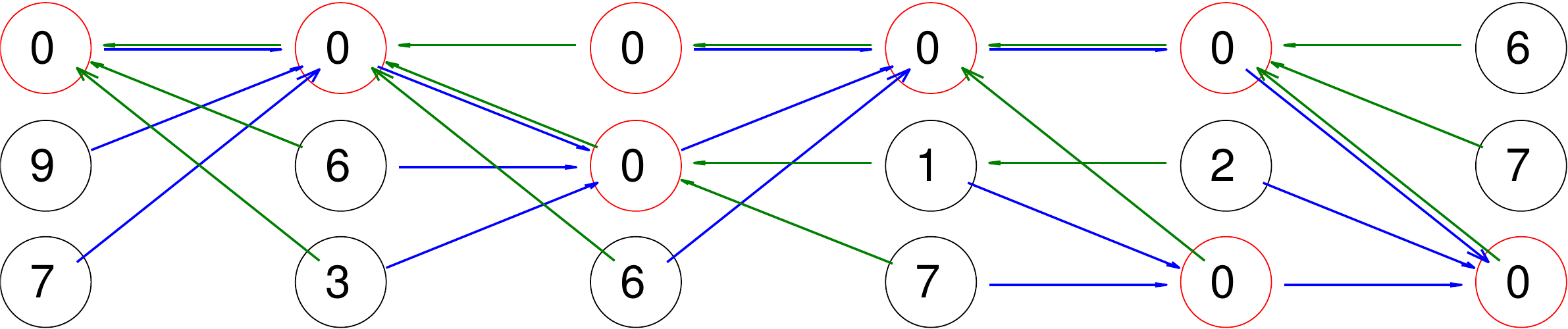}\end{tabular}\\
\small (c)\ \begin{tabular}{c}\includegraphics[width=0.9\linewidth]{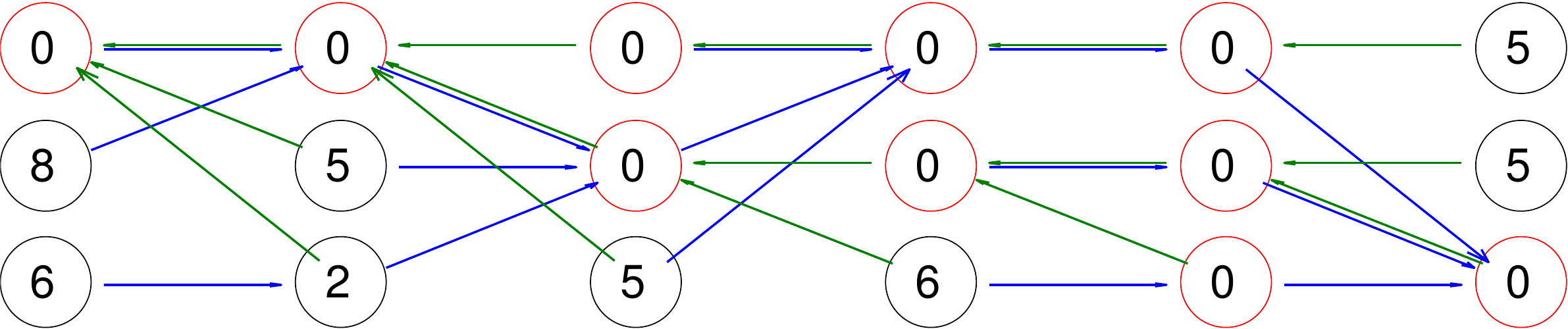}\end{tabular}\\
\end{tabular}
\caption{(a) Min-marginals (normalized by subtracting the value of the minimum) at vertices and arrows allowing to backtrack the optimal solution passing through a given vertex. (b), (c) min-marginals of $f-\lambda$ after one (resp. two) iterations of~\Algorithm{A:uniform-minorant} ($\varepsilon_1 = 1$ and $\varepsilon_2 = 1$). With each iteration the number of vertices having zero min-marginal strictly increases.
\label{fig:example-min-marg}}
\end{figure}
\par
\end{example}
\par
A basic performance test of \DMM with uniform minorants versus TRW-S is shown in~\Figure{fig:alg-uniform}. It demonstrates that the \DMM can be faster, when provided good minorants. The only problem is that determining the uniform minorant involves repeatedly solving minimum ratio path problems, plus there is a numerical instability in determining the support set of optimal solutions $O$.
%
%
\subsubsection{Iterative Minorants}
\begin{algorithm}[tb]
\KwIn{Chain subproblem $f$\;}
\KwOut{Minorant $\lambda$\;}
$\lambda :=0$\;
\For{$s = 1 \dots {\tt max\_pass}$}{
\For{$i = 1 \dots |V|$}{
	Compute min-marginal $m_i$ of $f-\lambda$ at $i$ dynamically, equations ~\eqref{msg-pass} and \eqref{min-marg-expression}\;
	$\lambda_i\ {+}{=}\ \gamma_s m_i$\;
}
Reverse the chain\;
}
\caption{Iterative Minorant\label{A:iterativr-minorant}}
\end{algorithm}
A simpler way to construct a maximal minorant would be to iteratively subtract from $f$ a portion of its min-marginals and accumulate them in the minorant, until all min-marginals of the reminder become zero. \Algorithm{A:iterativr-minorant} implements this idea. The portion of min-marginals drained from the reminder $f-\lambda$ to the minorant $\lambda$ in each iteration is controlled by $\gamma_s \in (0, 1]$. Reversing the chain efficiently alternates between the forward and the backward passes. For the last pass coefficient $\gamma_s$ is set to $1$ to ensure that the output minorant is maximal. \Figure{fig:alg-iterative} illustrates that this idea can perform well in practice.
\begin{figure}[th]
\centering
\includegraphics[width=\linewidth]{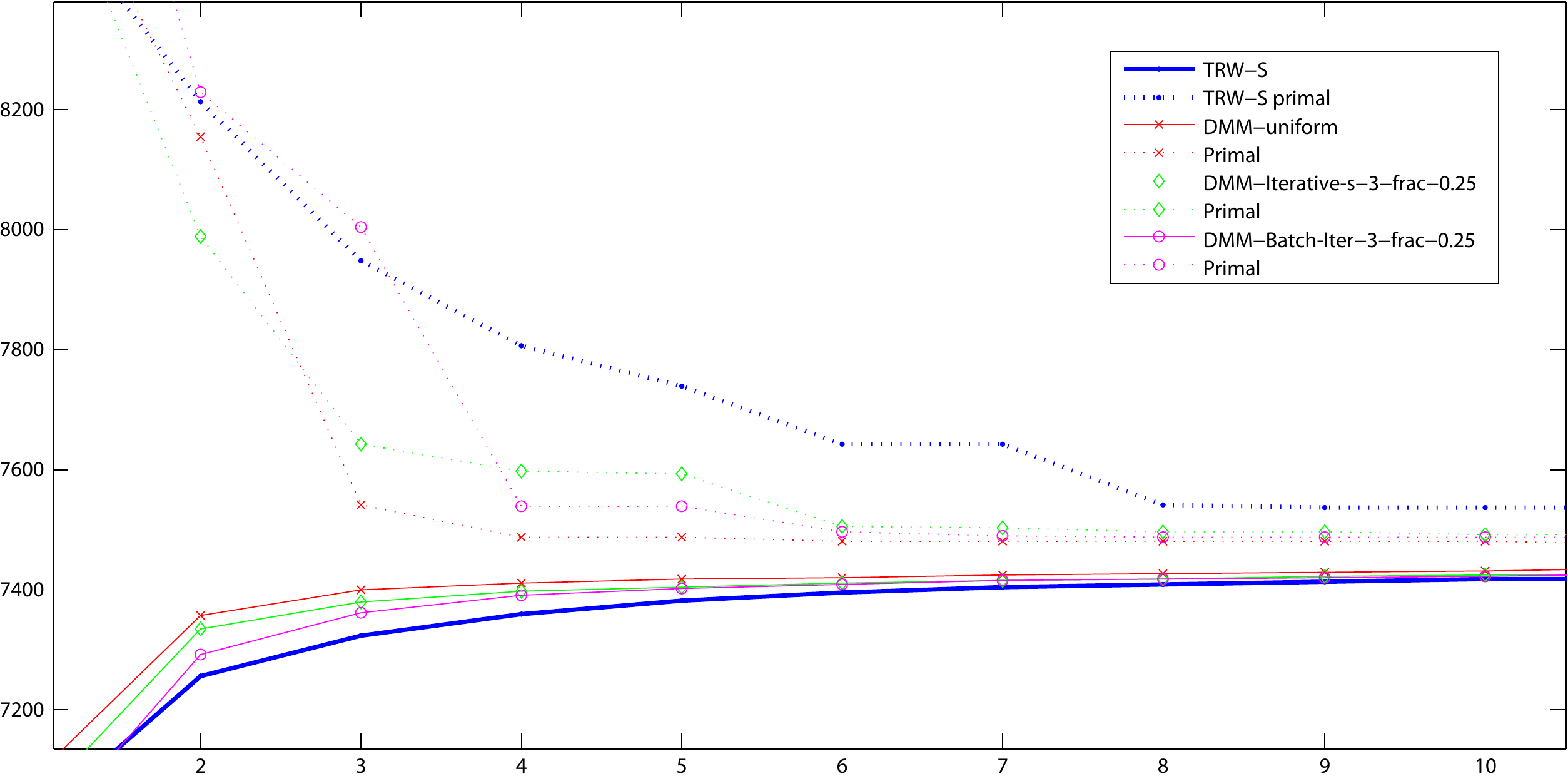}
\caption{Same setting as in \Figure{fig:alg-uniform}. The new plots show that Iterative minorants are not as good as uniform but still perform very well. Parameter ${\tt max\_pass}=3$ and  $\gamma_s = 0.25$ were used. The Batch Iterative method (Batch-Iter) runs forward-backward iterations in a smaller range, which is more cache-efficient and is also performing relatively well in this example.
\label{fig:alg-iterative}}
\end{figure}

\subsubsection{Hierarchical Minorants}
The idea of hierarchical minorants is as follows. Let $f$ be a one horizontal chain. We can break it into two subchains of approximately the same size, sharing a variable $x_i$ in the middle. By introducing a Lagrange multiplier over this variable, we can decouple the two chains. The value of the Lagrange multiplier can be chosen such that both subchains have exactly the same min-marginals in $x_i$. This makes the split uniform in a certain sense. Proceeding so we increase the amount of parallelism and hierarchically break the chain down to two-variable pieces, for which the minorant is computed more or less straightforwardly. This is the method used to obtain all visual experiments in the paper. Its more detailed benchmarking is left for future work.
%
%
%
\begin{algorithm}[!th]
\KwIn{Energy terms $f_i$, $f_j$, $f_{ij}$, messages $\varphi_{i-1,i}(x_i)$ and $\varphi_{j,j+1}(x_j)$ \;}
\KwOut{Messages for decorrellated chains: $\varphi_{ji}(x_i)$ and $\varphi_{ij}(x_j)$ \;}
\SetKwFunction{Msg}{{\ttfamily Msg}}
\tcc{Message from $j$ to $i$}
$\varphi_{ji}(x_i) := \Msg_{ji}(f_j + \varphi_{j,j+1})$\;
\tcc{Total min-marginal at $i$}
$m_i(x_i) := \varphi_{i-1,i}(x_i) + f_i(x_i) + \varphi_{ji}(x_i)$\;
\tcc{Share a half to the right}
$\varphi_{ij}(x_j) := \Msg_{ij}( m_i / 2 - \varphi_{ji} )$\;
\tcc{Bounce back what cannot be shared}
$\varphi_{ji}(x_i) := \Msg_{ji}( -\varphi_{ij} )$\label{bounce-back}\;
\SetKwProg{myproc}{Procedure}{}{}
  \myproc{$\Msg_{ij}(a)$}{
	\KwIn{Unary cost $a \in \Real^K$\;}
  \KwOut{Message from $i$ to $j$\;}
	\Return{$\varphi(x_j) := \min_{x_i \in \L} \big[ a(x_{i}) + f_{ij}(x_{i},x_{j})\big]$}\;
	}
\caption{\protect\anchor[{\texttt{Handshake}}]{Handshake}\label{A:handshake}}
\end{algorithm}
We detail now the simplest case when the chain has length two, \ie, the energy is given by $f_1(x_1)+f_{12}(x_1,x_2) + f_2(x_2)$.
The procedure to compute the minorant is as follows:
\begin{itemize}
\item Compute $m^f_1(x_1)$ and let $\lambda_1 := m^f_1(x_1) / 2$. \Ie, we subtract a half of the min-marginal in the first node.
\item Recompute the new min-marginal at node $2$: update the message $\varphi_{12}(x_2) := \Msg_{12}(f_1 - \lambda_1)$; Reassemble $m^{f-\lambda}_2(x_2) = \varphi_{12}(x_2) + f_2(x_2)$.
\item Take this whole remaining min-marginal to the minorant: let $\lambda_2 := m^{f-\lambda}_2(x_2)$.
\item Recompute the new min-marginal at node $1$: update the message $\varphi_{21}(x_1) := \Msg_{21}(f_2 - \lambda_2)$; It still may be non-zero. For example, if the pairwise term of $f$ is zero we recover the remaining half of the initial min-marginal at node $1$. Let $\lambda_1\ {+}{=}\ m^{f-\lambda}_1(x_1)$.
\end{itemize}
Importantly, the computation has been expressed in terms of message passing, and therefore can be implemented as efficiently.
The procedure fro the two-node case is straightforwardly generalized to longer chains. Let $ij$ be an edge in the middle of the chain. We compute left min-marginal at $i$, right min-marginal at $j$ and then apply the \Handshake procedure over the edge $ij$, defined in \Algorithm{A:handshake}. The procedure divides the slack between nodes $i$ and $j$ similarly to how it is described above for the pair. The result of this redistribution is encoded directly in the messages. 
The two subchains $1,\dots i$ and $j, \dots |\V|$ are ``decorrellated'' by the \Handshake and will not talk to each other further during the construction of the minorant. The left min-marginal for subchain $j, \dots |\V|$ at node $j+1$ is computed using update~\eqref{msg-pass} and so on until the middle of the subchain where a new \Handshake is invoked. The minorant is computed at the lowest level of hierarchy when the length of the subchain becomes two. The structure of the processing is illustrated in \Figure{fig:hierarchical-msgs}. It is seen that each level after the top one requires to send messages only for a half of nodes in total. Moreover, there is only a logarithmic number of level. It turns out that this procedure is not much more computationally costly than just computing min-marginals.
\begin{figure}
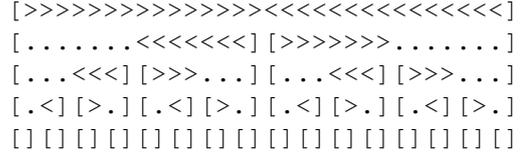

\centering
\texttt{[>>>>>>>>>>>>>>><<<<<<<<<<<<<<<]}\\
\texttt{[.......<<<<<<<][>>>>>>>.......]}\\
\texttt{[...<<<][>>>...][...<<<][>>>...]}\\
\texttt{[.<][>.][.<][>.][.<][>.][.<][>.]}\\
\texttt{[][][][][][][][][][][][][][][][]}
\caption{Messages passed in the construction of the hierarchical minorant for a chain of length 32. From top to bottom: level of hierarchical processing.
Symbols \texttt{>} and \texttt{<} denote message passing in the respective direction.
Brackets \texttt{[]} mark the limits of the decorrellated sub-chains at the current level. Dots denote places where the previously computed messages in the needed direction remain valid and need not be recomputed. Places where the two opposite messages meet correspond to the execution of the \Handshake procedure. The lowest level consists of 16 decorrellated chains of length 2 each.
\label{fig:hierarchical-msgs}}
\end{figure}
For example, to restore left min-marginal for the subchain $j,\dots |V|$, in node $i+1$ we 
%
%
%
%
\par
We conjecture that while iterative minorants may transfer only a geometric fraction of min-marginals in some cases, the hierarchical minorant is only by a constant factor inferior to the uniform one.
%
\subsection{Iteration Complexity}
The bottleneck in a fast implementation of dual algorithms are the memory access operations. This is simply because there is a big cost data volume that needs to be scanned in each iteration plus messages have to be red and written in TRW-S as well as in out \Algorithm{DMM} (dual variables $\lambda$). 
We therefore will assess complexity in terms of memory access operations and ignore the slightly higher arithmetic complexity of our minorants. 
\par
For TRW-S the accesses per pixel are:
\begin{itemize}
\item read all incoming messages (4 access);
\item read data term (1 access);
\item write out messages in the pass direction (2 accesses).
\end{itemize}
The cache can potentially amortize writing messages and reading them back in the next scan line, in which case the complexity could be counted as 5 accesses per pixel. However, currently only CPU cache is big enough for this, while multiprocessors in GPU have relatively small cache divided between many parallel threads.
\par
For the iterative minorant we have 3 forward-backward passes reading the data cost, the reverse message and writing the forward message (3*2*3 accesses), the last iteration writes $\lambda$ and not the message. Some saving is possible with a small cache set at a cost of more computations. Computing the hierarchical minorant as described in \Figure{fig:hierarchical-msgs} for a chain of length 2048, assuming that chunks of size $8$ already fit in the fast memory (registers + shared memory) has the following complexity. Reading data costs and writing messages until length 8 totals to $2 + \log_2(2048/8) / 2 = 6$ accesses. Reading messages is only required at \Handshake points and needs to be counted only until reaching length 8. Writing $\lambda$ adds one more access. These estimates are summarized in \cref{tab:mem-complexity}.
\begin{table}[htb]
\setlength{\tabcolsep}{3pt}
\begin{tabular}{c|c|c|c}
TRW-S & Iterative & Naive BCD & Hierarchical \\
\hline
7(5) & 18(8) & 5(4) & 7
\end{tabular}
\caption{Memory accesses per pixel in TRW-S and \DMM with variants of minorants. Naive BCD here means just computing min-marginals.\label{tab:mem-complexity}}
\end{table}
%
%



%
\end{document}